\documentclass[smallextended]{mysvjour3} 

\RequirePackage{fix-cm}
\smartqed % flush right qed marks

\usepackage{cite}
\usepackage{array}
\usepackage[letterpaper,centering,top=1.2in,bottom=1.2in]{geometry}
\usepackage[pdftex]{graphicx}
\usepackage[cmex10]{amsmath}
\usepackage{amsfonts,amssymb,bbm,verbatim}
\usepackage{mathrsfs,dsfont} % for \mathscr{} and \mathds{} fonts
\usepackage{hyperref}
\hypersetup{colorlinks = true, citecolor = blue, urlcolor = blue}
\usepackage[all]{hypcap} % for going to the top of an image when a figure reference is clicked
\usepackage[usenames,dvipsnames]{color}
\usepackage[utf8]{inputenc} % allow utf-8 input
\usepackage{eqparbox} % for correct spacing in overset equations
\usepackage{booktabs,multirow} % for professional quality tables
\usepackage{xcolor} % to color blocks of text
\usepackage{algorithm}
\usepackage{algpseudocode} % uses algorithmicx package

\makeatletter
\algnewcommand{\StatexInd}[1]{\Statex \hskip\ALG@thistlm #1}

\allowdisplaybreaks % to allow display breaks in align environment
\newcommand{\R}{\mathbb{R}}
\newcommand{\N}{\mathbb{N}}
\newcommand{\Z}{\mathbb{Z}}
\renewcommand{\O}{\mathcal{O}}

\newcommand{\A}{\mathcal{A}}
\newcommand{\D}{\mathcal{D}}
\newcommand{\G}{\mathcal{G}}
\newcommand{\B}{\mathcal{B}}
\renewcommand{\L}{\mathcal{L}}

\newcommand{\I}{\mathbbm{1}}

\newcommand{\E}{\mathbb{E}}
\newcommand{\T}{\mathrm{T}}
\newcommand{\F}{\mathrm{F}}
\newcommand{\argmin}[1]{\underset{#1}{\operatorname{arg}\,\operatorname{min}}\;}
\DeclareMathOperator\diff{{d \!}}
\DeclareMathOperator\tr{{tr}}
\DeclareMathOperator\poly{{poly}}
\DeclareMathOperator\polylog{{polylog}}

\newcommand{\secref}[1]{Section~\ref{#1}}
\newcommand{\appref}[1]{Appendix~\ref{#1}}

\newcommand{\thmref}[1]{Theorem~\ref{#1}}
\newcommand{\propref}[1]{Proposition~\ref{#1}}
\newcommand{\lemref}[1]{Lemma~\ref{#1}}

\newcommand{\algoref}[1]{Algorithm~\ref{#1}}
\newcommand{\tabref}[1]{Table~\ref{#1}}

\begin{document}

\title{Gradient-Based Empirical Risk Minimization using Local Polynomial Regression\thanks{The author ordering is alphabetical.}
}

\author{Ali~Jadbabaie \and
				Anuran~Makur \and
				Devavrat~Shah
}
\authorrunning{Jadbabaie, Makur, and Shah} 

\institute{The authors are with the Laboratory for Information and Decision Systems and the Institute for Data, Systems, and Society, Massachusetts Institute of Technology, Cambridge, MA 02139, USA. E-mails: jadbabai@mit.edu; a\_makur@mit.edu; devavrat@mit.edu % \email{}
}

%\date{November 2020}

\maketitle

\begin{abstract}
In this paper, we consider the widely studied problem of empirical risk minimization (ERM) of strongly convex and smooth loss functions using iterative gradient-based methods. A major goal of this literature has been to compare different prototypical algorithms, such as batch gradient descent (GD) or stochastic gradient descent (SGD), by analyzing their rates of convergence to $\epsilon$-approximate solutions. For example, the (first order) oracle complexity of GD is $O(n \log(\epsilon^{-1}))$, where $n$ is the number of training samples. When $n$ is large, this can be prohibitively expensive in practice, and SGD is preferred due to its oracle complexity of $O(\epsilon^{-1})$. Such standard analyses only utilize the smoothness of the loss function in the parameter being optimized. In contrast, we demonstrate that when the loss function is smooth in the data, we can learn the oracle at every iteration and beat the oracle complexities of both GD and SGD in important regimes. Specifically, at every iteration, our proposed algorithm first performs local polynomial regression with a virtual batch of data points to learn the gradient of the loss function, and then estimates the true gradient of the ERM objective function. We establish that the oracle complexity of our algorithm scales like $\tilde{O}((p \epsilon^{-1})^{d/(2\eta)})$ (neglecting sub-dominant factors), where $d$ and $p$ are the data and parameter space dimensions, respectively, and the gradient of the loss function is assumed to belong to a $\eta$-H\"{o}lder class with respect to the data, not the parameter. Our proof extends the analysis of local polynomial regression in non-parametric statistics to provide supremum norm guarantees for interpolation in multivariate settings, and also exploits tools from the inexact GD literature. Unlike the complexities of GD and SGD, the complexity of our method depends on $d$ and $p$. However, we elucidate that when the data dimension $d$ is small and the loss function exhibits modest smoothness in the data $\eta = \Theta(d)$, our algorithm beats GD and SGD in oracle complexity for a very broad range of $p$ and $\epsilon$. 

\keywords{Gradient descent \and Empirical risk minimization \and Local polynomial regression \and Oracle complexity}
\end{abstract}

\section{Introduction}
\label{Introduction}

\emph{Empirical risk minimization} (ERM) is one of the mainstays of contemporary machine learning. Indeed, training tasks such as classification, regression, or representation learning using deep neural networks, can all be formulated as specific instances of ERM. In this paper, we consider gradient-based iterative optimization methods that are used to perform ERM, such as batch \emph{gradient descent} (GD) \cite{Nesterov2004}, \emph{stochastic gradient descent} (SGD) \cite{Nemirovskietal2009}, and their refinements, e.g., \cite{Nesterov1983,Dekeletal2012,JohnsonZhang2013}. A significant portion of this literature is concerned with analyzing the convergence rates of different iterative algorithms and determining which ones have smaller oracle complexities, i.e., smaller number of gradient queries required to obtain approximate minimizers, under various levels of convexity, smoothness, sampling, and other assumptions (see, e.g., \cite{Bubeck2015,Netrapalli2019} and the references therein). We focus on the widely studied setting of strongly convex loss functions. In this case, as depicted in \tabref{Table: Convergence rates}, the (first order) oracle complexity of GD is known to be $O(n \log(\epsilon^{-1}))$ \cite[Theorem 2.1.15]{Nesterov2004}, where $n$ is the number of training samples in the ERM objective function, and $\epsilon > 0$ is the desired approximation accuracy of the solution. In contrast, it is well-known that SGD has an oracle complexity of $O(\epsilon^{-1})$ \cite{Nemirovskietal2009}. Since $n$ is often extremely large in modern instances of ERM, SGD is theoretically preferable to GD. This is corroborated by the ubiquitous use of SGD and its variants for ERM in practice. Our goal is to improve upon the oracle complexities of GD, SGD, and their variants by exploiting additional information about the loss functions that has heretofore been neglected in the literature.

Recall that the general unconstrained ERM problem is usually abstracted as, cf. \cite[Section 1]{Bubeck2015}, \cite{HastieTibshiraniFriedman2009}:
\begin{equation}
\label{Eq: Standard formulation}
\min_{\theta \in \R^p}{\frac{1}{n} \sum_{i = 1}^{n}{f_i(\theta)}} \, ,
\end{equation}
where each $f_i : \R^p \rightarrow \R$ represents the loss function corresponding to a different sample of training data. (Note that this formulation, modulo the extra scaling $n^{-1}$, is also utilized in \emph{distributed optimization}, where each $f_i$ represents the local objective function corresponding to an individual agent in a network of $n$ agents \cite{Nedic2015}.) As suggested by the form of \eqref{Eq: Standard formulation}, typical analysis of iterative gradient-based algorithms, such as GD or SGD, makes no assumptions about $f_i$ as a function of $i$. However, the loss function often varies quite smoothly with the data in ERM. For example, given training data $\big\{(x^{(i)},y_i) \in \R^p \times \R : i \in [n]\big\}$ (with $[n] \triangleq \{1,\dots,n\}$), the classical least squares formulation of \emph{linear regression} with regularization is (see, e.g., \cite[Chapter 3]{HastieTibshiraniFriedman2009}): 
\begin{equation}
\label{Eq: Lin Reg}
\min_{\theta \in \R^p}{\frac{1}{n} \sum_{i = 1}^{n}{\left(y_i - \theta^{\T} x^{(i)}\right)^{\!2} + \mathcal{R}(\theta)}} \, ,
\end{equation} 
where the objective is to find the optimal coefficients $\theta \in \R^p$, and $\mathcal{R}(\theta)$ is some regularization term, e.g., ridge or lasso. Letting $f(x,y;\theta) = \big(y - \theta^{\T} x\big)^2 + \mathcal{R}(\theta)$ for $x,\theta \in \R^p$ and $y \in \R$, we see that each $f_i(\theta) = f(x^{(i)},y_i;\theta)$ in formulation \eqref{Eq: Standard formulation} for this problem. Clearly, $f$ is very smooth as a function of the data $(x,y)$, regardless of the regularization $\mathcal{R}(\theta)$. This observation that loss functions are often smooth in the data holds in more complex learning scenarios too. For instance, when training a deep neural network with a smooth loss, e.g., cross-entropy loss, and smooth activation functions, e.g., logistic (sigmoid) functions, the induced loss function for the ERM objective function is infinitely differentiable in both the data and the network's weight parameters \cite[Chapter 5]{Bishop2006}. Since such smoothness in the data has not been exploited by modern analysis of iterative gradient-based optimization algorithms for \eqref{Eq: Standard formulation} (see \secref{Related Literature}), the main thrust of this paper is to speed up algorithms for ERM using this additional smoothness in the data. 

To this end, in the remainder of this work, we consider the following problem formulation of ERM. For $d,n \in \N$, assume that we are given $n$ samples of $d$-dimensional training data $\D = \{x^{(i)} \in [h^{\prime},1-h^{\prime}]^d : i \in [n]\}$, where $h^{\prime} > 0$ is an arbitrarily small parameter (to be determined later; see \eqref{Eq: Data bound}), and each training sample belongs to the hypercube $[h^{\prime},1-h^{\prime}]^d$ without loss of generality. (This compactness assumption affords us some analytical tractability in the sequel.) Note that when $d \geq 2$, the first $d-1$ elements of a sample can be perceived as a $(d-1)$-dimensional feature vector, and the last element of the sample can be construed as the corresponding continuous-valued label (as in the linear regression example above). For a given $p \in \N$, fix the (parametrized) loss function $f : [0,1]^{d} \times \R^p \rightarrow \R$, $f(x;\theta)$, where $\theta = (\theta_1,\dots,\theta_p) \in \R^p$ is a $p$-dimensional parameter vector. Now define the ERM objective function $F : \R^p \rightarrow \R$ as follows:
\begin{equation}
\label{Eq: ERM Objective Function}
\forall \theta \in \R^p, \enspace F(\theta) \triangleq \frac{1}{n} \sum_{i = 1}^{n}{f(x^{(i)};\theta)} 
\end{equation}
which is the empirical risk, i.e., the empirical expectation of the loss function with respect to the training data $\D$. With this objective function, we consider the unconstrained minimization problem:
\begin{equation}
\label{Eq: ERM}
F_* \triangleq \inf_{\theta \in \R^p}{F(\theta)} \, ,
\end{equation}
where $F_* \in \R$ denotes the minimum value. In contrast to the canonical formulation \eqref{Eq: Standard formulation}, our formulation in \eqref{Eq: ERM Objective Function} and \eqref{Eq: ERM} explicitly states the dependence of the loss function on the data.

\begin{table}[t]
\caption{\textbf{Oracle complexities of GD, SGD, and our algorithm for strongly convex loss functions:} Here, $\epsilon > 0$ is the approximation accuracy, $p \in \N$ is the dimension of the parameter space, $d \in \N$ is the dimension of the data space, $n \in \N$ is the number of training data samples, $\eta > 0$ is the H\"{o}lder exponent determining the smoothness of the loss function in the data space, and $C(d,\eta) = C_{\mu,L_1,L_2,b,c}(d,l)$ is given in \eqref{Eq: Scaling in Main Theorem}.\\}
\label{Table: Convergence rates}
\centering
\begin{tabular}{lc}
\toprule
\textbf{Algorithm} & \textbf{Oracle complexity} \\
\toprule
GD & $O(n\log(\epsilon^{-1}))$ \cite{Nesterov2004} \\
\midrule
SGD & $O(\epsilon^{-1})$ \cite{Nemirovskietal2009} \\
\midrule
LPI-GD & \textcolor{blue}{$O(C(d,\eta) (p \epsilon^{-1})^{d/(2\eta)} \log(p\epsilon^{-1}))$} (\thmref{Thm: Oracle Complexity}) \\
\bottomrule
\end{tabular}
\end{table}
 
Under the smoothness and strong convexity assumptions delineated in \secref{Smoothness Assumptions and Approximate Solutions}, we make the following main contributions:
\begin{enumerate}
\item We propose a \emph{new inexact gradient descent} algorithm to compute approximate solutions of \eqref{Eq: ERM} in \secref{Description of Algorithm} (see \algoref{Algorithm: LPI-GD}). The main innovation of this algorithm is to use the smoothness of the loss function $f$ in the data to learn the gradient oracle of $f$ at every iteration by performing local polynomial regression based on oracle queries over a set of virtual data points. 
\item We derive the iteration and oracle complexities of this algorithm in \propref{Prop: Number of Iterations} and \thmref{Thm: Oracle Complexity}, respectively, in \secref{Analysis of Algorithm}. In particular, we show that the oracle complexity scales like $\tilde{O}((p \epsilon^{-1})^{d/(2\eta)})$ (neglecting sub-dominant factors) as indicated in the last row of \tabref{Table: Convergence rates}, where $\eta > 0$ denotes the H\"{o}lder exponent determining the smoothness of the gradient of the loss function with respect to the data (see \secref{Smoothness Assumptions and Approximate Solutions}). Although we focus on strongly convex loss functions to leverage simple convergence results on inexact gradient descent from the literature, analogous analysis holds for larger classes of loss functions, e.g., when $f(x;\cdot):\R^p \rightarrow \R$ is non-convex. 
\item Furthermore, we show in \propref{Prop: Comparison to GD and SGD} in \secref{Analysis of Algorithm} that if the data dimension is sufficiently small, e.g., $d = O(\log\log(n))$ (as in healthcare data analytics or control applications; see \secref{Related Literature}), and the loss function $f$ exhibits modest smoothness in the data, i.e., $\eta = \Theta(d)$, then the oracle complexity for our method beats the oracle complexities (or more precisely, the oracle complexity bounds) of both GD and SGD for any $p$ that grows polynomially in $n$ and any $\epsilon$ that decays polynomially in $n^{-1}$. This demonstrates that smoothness of loss functions in the data can be successfully exploited in gradient-based algorithms for ERM to significantly lower their oracle complexity.
\item Finally, in order to analyze the convergence rate of our algorithm, we also generalize the technique of local polynomial regression, or more precisely, local polynomial interpolation (since there is no noise in our problem), in non-parametric statistics to the multivariate setting in \thmref{Thm: Chebyshev Norm Interpolation} in \secref{Local Polynomial Regression}. Although it is implied in many standard expositions that univariate local polynomial regression can be generalized to multiple variables (see, e.g., \cite{FanGijbels1996}, \cite{Tsybakov2009}, or \cite{Wasserman2019}), the multivariate case requires significantly more care than the univariate case. Consequently, a detailed and rigorous analysis of local polynomial interpolation is difficult to find in the literature. Therefore, we believe that \thmref{Thm: Chebyshev Norm Interpolation} may be of independent interest in statistics.
\end{enumerate}

Let us briefly outline the rest of this paper. We discuss some related literature in \secref{Related Literature} that further distinguishes our contributions from known results, and conveys that our results apply to practical machine learning settings. Then, we introduce required notation and definitions, formally state key assumptions, and briefly elaborate on the standard framework of oracle complexity in optimization theory in Sections \ref{Notation}, \ref{Smoothness Assumptions and Approximate Solutions}, and \ref{Oracle Complexity}, respectively. In \secref{Main Results and Discussion}, we present our main results as enumerated above. We prove our results on local polynomial regression in \secref{Proofs for Local Polynomial Interpolation}, and carry out the convergence analysis of our algorithm in \secref{Proofs of Convergence Analysis}. Finally, we conclude our discussion and propose future research directions in \secref{Conclusion}.

\subsection{Related Literature}
\label{Related Literature}

There is an enormous literature concerning the analysis of convergence rates of exact and inexact gradient-based algorithms that perform the optimization in \eqref{Eq: Standard formulation}. We do not attempt to survey this entire field here, and instead only mention some noteworthy works. For example, in our strongly convex setting of interest, the convergence analysis of GD can be found in the standard text \cite{Nesterov2004}, and corresponding convergence analysis of SGD can be found in \cite{Nemirovskietal2009} (also see \cite{Netrapalli2019} and the references therein). Moreover, various refinements and improvements of these basic iterative algorithms have been analyzed in the literature. Notable examples include the famous \emph{accelerated} (or \emph{momentum} based) GD \cite{Nesterov1983}, mini-batch SGD methods \cite{Dekeletal2012}, and \emph{variance reduced} SGD methods \cite{JohnsonZhang2013} (also see the references therein). We also refer readers to \cite{Bubeck2015} for a unified treatment of many of the aforementioned approaches. More generally, inexact gradient descent methods for ERM have been analyzed extensively as well\textemdash see, e.g., \cite{FriedlanderSchmidt2012,SoZhou2017} and the references therein for work on the strongly convex case. (Note that SGD can be perceived as an example of an inexact GD method, which uses unbiased estimates of the true gradient of the ERM objective function at every iteration.) This paper and most of the aforementioned references focus on obtaining upper bounds on oracle complexity. Conversely, several authors have also established fundamental lower bounds, cf. \cite{Agarwaletal2009,Carmonetal2019} and the references therein. As mentioned earlier, to our knowledge, all of this literature concerns the formulation in \eqref{Eq: Standard formulation}. Hence, these methods do not exploit smoothness of loss functions in data by learning gradients in every iteration as in our proposed method.

On the other hand, motivated by the utility of gradient information for various tasks like \emph{variable selection} and determining \emph{coordinate covariations}, cf. \cite{MukherjeeZhou2006,MukherjeeWu2006}, the problem of learning gradients in the context of supervised learning problems, such as classification and regression, has received a lot of attention from the machine learning community. Specifically, most of this literature aims to \emph{simultaneously} learn a classification or regression function along with its gradient from training data. (Again, we only mention some noteworthy works here.) For example, the text \cite{FanGijbels1996} presents the classical theory of gradient estimation using local polynomial fitting, and \cite{DeBrabanteretal2013} presents a gradient estimation approach that uses local polynomial regression, but does not estimate the regression function. In a related vein, \cite{DelecroixRosa1996} studies kernel methods for gradient estimation. Various other methods for learning gradients have also been analyzed, such as \emph{reproducing kernel Hilbert space} methods \cite{MukherjeeZhou2006,MukherjeeWu2006}, regression \emph{splines} \cite{ZhouWolfe2000}, and \emph{nearest neighbor} methods \cite{AussetClemenconPortier2020}. However, the vast majority of these techniques have not been considered in the context of optimization for machine learning. 

Two recent exceptions where gradient estimation has been applied to optimization are \cite{Wangetal2018} and \cite{AussetClemenconPortier2020}. In particular, the authors of \cite{Wangetal2018} illustrate that learning gradients leads to desirable convergence rates for \emph{derivative-free} optimization algorithms, which only use zeroth order oracles (i.e., these oracles only output function values, not gradient values). Similarly, \cite{AussetClemenconPortier2020} portrays the effectiveness of its nearest neighbor based gradient estimator for derivative-free optimization. It is worth noting that both these works analyze optimization of a single objective function (rather than a sum as in \eqref{Eq: Standard formulation}). Hence, the results derived in these papers are not particularly insightful in our ERM setting in \eqref{Eq: ERM Objective Function} and \eqref{Eq: ERM}, where we seek to learn gradients based on a first order oracle by exploiting the smoothness of $f(x;\theta)$ in the data $x$. 

Lastly, as we remarked earlier, we establish in \propref{Prop: Comparison to GD and SGD} that the oracle complexity of our algorithm beats GD and SGD for a broad range of values of $p$ and $\epsilon$, if the data dimension $d$ is sufficiently small, e.g., $d = O(\log\log(n))$. Since we utilize multivariate local polynomial regression in our algorithm, a ``small $d$ assumption'' is theoretically unavoidable due to the \emph{curse of dimensionality}, cf. \cite[Section 7.1]{FanGijbels1996}. We close this discussion of related literature by further explaining why this assumption is also very reasonable in applications. Indeed, many application domains have problems where the number of data attributes is small. For example, in the healthcare data analytics space, electronic health records often have very few features per patient, e.g., temperature, blood pressure, initial diagnosis, and medications, but machine learning techniques are still widely used to build models for diagnosing diseases or predicting risk (cf. \cite{Choietal2018} and its references). As another example, in the problem of system identification for robotics and control applications, neural networks have been used to build non-linear state-space models (cf. \cite{NechybaXu1994,Nelles2001} and their references). The number of inputs of these networks is often very small when the model order is known or assumed to be low. On the other hand, in many applications with high-dimensional ``raw data,'' data scientists believe that the useful information in the data lives in a low-dimensional manifold. This belief is the basis of a myriad of \emph{dimensionality reduction} techniques, such as principal component analysis \cite{Pearson1901,Hotelling1933}, canonical correlation analysis \cite{Hotelling1936}, Laplacian eigenmaps \cite{BelkinNiyogi2001} or diffusion maps \cite{CoifmanLafon2006}, and modal decompositions (or correspondence analysis), cf. \cite{Huangetal2019}, \cite[Chapter 4]{Makur2019}, and the references therein. As a result, many high-dimensional data analysis pipelines proceed by first pre-processing the data using pertinent dimensionality reduction techniques, and then using the resulting low-dimensional features for other learning tasks. In such high-dimensional settings, our contributions in this paper pertain to ERM using the extracted low-dimensional features. 

\subsection{Notation}
\label{Notation}

We briefly introduce some notation that is used throughout this paper. Let $\N \triangleq \{1,2,3,\dots\}$ and $\Z_+ \triangleq \{0,1,2,3,\dots\}$ denote the sets of natural numbers and non-negative integers, respectively. As mentioned earlier, for any $n \in \N$, let $[n] \triangleq \{1,\dots,n\}$ be the set of positive integers less than $n+1$. In the sequel, we analyze a $d$-variate problem setting with $d \in \N$. So, we briefly introduce the standard multi-index notation. For any $d$-tuple of non-negative integers $s = (s_1,\dots,s_d) \in \Z_+^d$, we let $|s| = s_1 + \cdots + s_d$, $s! = s_1! \cdots s_d!$, and $x^s = x_1^{s_1} \cdots x_d^{s_d}$ for $x \in \R^d$. Moreover, given a continuously differentiable function $f : \R^d \rightarrow \R$, $f(x_1,\dots,x_d)$, we write $\nabla^s f = \partial^{|s|} f/(\partial x_1^{s_1} \cdots \partial x_d^{s_d})$ to denote its $s$th partial derivative, and $\nabla_x f = \big[\frac{\partial f}{\partial x_1} \, \cdots \, \frac{\partial f}{\partial x_d}\big]^{\T}$ to denote its gradient. We will sometimes find it convenient to index certain vectors or matrices using the set $\{s \in \Z_+^d : |s| \leq l\}$. In these cases, to help with calculations, we will assume that these vectors or matrices are written out with the indices in lexicographical order (i.e., for any $r,s \in \Z_+^d$ with $r \neq s$, we say that $r < s$ if $r_j < s_j$ for the first position $j \in [d]$ where $r_j \neq s_j$).

Furthermore, for any vector $x \in \R^n$ with $n \in \N$, we let $x_i$ be the $i$th entry of $x$, where $i \in [n]$ (or more generally, $i$ belongs to an appropriate index set), and we let $\|x\|_{p}$ be the usual $\ell^p$-norm of $x$ with $p \in [1,+\infty]$. Similarly, for any matrix $A \in \R^{m \times n}$ with $m,n \in \N$, we let $A_{i,j}$ be the $(i,j)$th element of $A$, where $i \in [m]$ and $j \in [n]$ (or more generally, $i$ and $j$ belong to appropriate index sets), and we let $\|A\|_{\F}$ be the Frobenius norm of $A$. In addition, $A^{\T}$ represents the transpose of $A$, and when $m = n$, $A^{-1}$ represents the inverse of $A$ if it is invertible, $\tr(A)$ represents the trace of $A$, $\det(A)$ represents the determinant of $A$, and $\lambda_{\mathsf{min}}(A)$ represents the minimum eigenvalue of $A$. We also let $I_n$ be the $n \times n$ identity matrix. For any pair of symmetric matrices $A,B \in \R^{n \times n}$, we use $\succeq$ to denote the L\"{o}wner partial order and write $A \succeq B$ to mean that $A - B$ is positive semidefinite. 
  
Finally, we introduce some useful miscellaneous notation. We let $\lceil \cdot \rceil$ denote the ceiling function, $\I\{\cdot\}$ denote the indicator function that equals $1$ if its input proposition is true and equals $0$ otherwise, $\exp(\cdot)$ and $\log(\cdot)$ denote the natural exponential and logarithm functions with base $e$, respectively, and $\E[\cdot]$ denote the expectation operator, where the underlying probability measure can be easily inferred from context. Throughout this paper, we will also utilize standard Bachmann-Landau asymptotic notation, e.g., $O(\cdot)$, $\Theta(\cdot)$, $\omega(\cdot)$, etc., where the underlying variable that tends to infinity will be clear from context. Moreover, we will let $\poly(n)$ and $\polylog(n)$ denote polynomial and poly-logarithmic scaling in the variable $n$, respectively.

\subsection{Smoothness Assumptions and Approximate Solutions}
\label{Smoothness Assumptions and Approximate Solutions}

In order to perform tractable convergence analysis of gradient-based iterative algorithms that solve \eqref{Eq: ERM}, we must formally define the problem class by specifying any smoothness conditions on the loss function, the appropriate notion of ``approximate solutions,'' and the oracle model of computational complexity used \cite[Section 1.1]{Nesterov2004}. This section and \secref{Oracle Complexity} precisely define our problem class.

Recall our formulation where we are given $n$ samples of $d$-dimensional training data $\D$ and a loss function $f : [0,1]^d \times \R^p \rightarrow \R$, $f(x,\theta)$, and we consider the ERM problem in \eqref{Eq: ERM Objective Function} and \eqref{Eq: ERM}. Let $\epsilon > 0$ be any small constant representing the desired level of approximation accuracy. We begin by imposing the following assumptions on our loss function:
\begin{enumerate}
\item \textbf{Smoothness in parameter:} There exists $L_1 > 0$ such that for all fixed $x \in [0,1]^{d}$, the gradient of $f$ with respect to $\theta$, denoted $\nabla_{\theta} f(x;\cdot) : \R^p \rightarrow \R^p$, exists and is $L_1$-\emph{Lipschitz continuous} as a function of the parameter vector:
\begin{equation}
\forall \theta_1,\theta_2 \in \R^p, \enspace \left\|\nabla_{\theta} f(x;\theta_1) - \nabla_{\theta} f(x;\theta_2)\right\|_2 \leq L_1 \left\|\theta_1 - \theta_2\right\|_2 \, . 
\end{equation}
\item \textbf{Smoothness in data:} There exist $\eta > 0$ and $L_2 > 0$ such that for all fixed $\vartheta \in \R^{p}$ and for all $i \in [p]$, the $i$th partial derivative of $f$ (with respect to $\theta_i$) at $\vartheta$, denoted $g_{i}(\cdot;\vartheta) \triangleq \frac{\partial f}{\partial \theta_i}(\cdot;\vartheta) : [0,1]^d \rightarrow \R$, belongs to the $(\eta,L_2)$-\emph{H\"{o}lder class} as a function of the data vector $x$ (cf. \cite[Definition 1.2]{Tsybakov2009}). By definition, this means that each $g_{i}(\cdot;\vartheta) : [0,1]^d \rightarrow \R$ is $l = \lceil \eta \rceil - 1$ times differentiable, and for every $s = (s_1,\dots,s_d) \in \Z_+^d$ such that $|s| = l$, we have
\begin{equation}
\label{Eq: Holder condition}
\forall y_1,y_2 \in [0,1]^d, \enspace \left|\nabla^s g_i(y_1;\vartheta) - \nabla^s g_i(y_2;\vartheta)\right| \leq L_2 \left\|y_1 - y_2\right\|_1^{\eta - l} \, . 
\end{equation}
\item \textbf{Strong convexity:} There exists $\mu > 0$ such that for all fixed $x \in [0,1]^{d}$, the map $f(x;\cdot):\R^p \rightarrow \R$ is $\mu$-\emph{strongly convex} as a function of the parameter vector:
\begin{equation}
\forall \theta_1,\theta_2 \in \R^p, \enspace f(x;\theta_1) \geq f(x;\theta_2) + \nabla_{\theta} f (x;\theta_2)^{\T} (\theta_1 - \theta_2) + \frac{\mu}{2} \left\|\theta_1 - \theta_2\right\|_2^2 \, . 
\end{equation}
Moreover, as additional notation, we let $\sigma \triangleq L_1/\mu$ represent the \emph{condition number} of the problem.
\end{enumerate}

Under these assumptions, we define an \textbf{$\epsilon$-approximate solution} to the optimization problem \eqref{Eq: ERM} as any parameter vector $\theta^* \in \R^p$ that satisfies
\begin{equation}
\label{Eq: Minimum of convex function}
F(\theta^*) - F_* \leq \epsilon \, .
\end{equation}
We make two pertinent remarks at this point. Firstly, it can be argued based on the above strong convexity assumption (as shown in \secref{Proof of Prop Number of Iterations}) that $F_*$ is finite and achievable by a unique global minimizer of $F$; in particular, the infimum in \eqref{Eq: ERM} may be replaced by a minimum. Secondly, both the first and third assumptions are standard in the optimization literature, cf. \cite{Nesterov2004,Bubeck2015}. As noted earlier, our main idea is the observation that the second assumption, that gradients of loss functions are also smooth in the data, holds in many machine learning applications. For example, in the regularized linear regression example in \eqref{Eq: Lin Reg}, $\nabla_{\theta} f$ is clearly smooth with respect to the data (assuming $\mathcal{R}$ is differentiable). 

\subsection{Oracle Complexity}
\label{Oracle Complexity}

Finally, we briefly introduce the formalism of ``first order oracle complexity'' upon which our analysis will be based. The notion of oracle complexity was first introduced in \cite{NemirovskiiYudin1983} to circumvent the difficulties faced by vanilla computational complexity theory style analysis of convex programming (cf. \cite[Section 1]{Agarwaletal2009}). As elucidated in \cite[Section 1.1]{Nesterov2004}, although ``arithmetical complexity,'' which accounts for the running time of all basic arithmetic operations executed by an algorithm, is a more realistic measure of computational complexity, the idea of oracle (or ``analytical'') complexity, i.e., the total number of oracle calls made by an algorithm, is a more convenient abstraction for the analysis of iterative optimization methods. Indeed, the running time of oracle queries is often the main bottleneck in arithmetical complexity. 

Observe that any iterative algorithm that approximately solves optimization problems belonging to the class of ERM problems defined in \eqref{Eq: ERM Objective Function} and \eqref{Eq: ERM}, under the assumptions in \secref{Smoothness Assumptions and Approximate Solutions}, requires two inputs: a) the training data $\D$, and b) a ``description'' of the specific loss function $f$ under consideration. To fulfill the latter requirement, we will assume that such algorithms have access to a local \emph{first order oracle} $\O : [0,1]^d \times \R^p \rightarrow \R^p$ defined as:
\begin{equation}
\label{Eq: Oracle}
\forall x \in [0,1]^d, \, \forall \vartheta \in \R^p, \enspace \O(x,\vartheta) \triangleq \nabla_{\theta} f(x;\vartheta) 
\end{equation} 
which returns the value of the gradient $\nabla_{\theta} f$ with respect to $\theta$ for any input query pair $(x,\vartheta)$. (Note that sometimes the oracle is defined to produce values of both the loss function $f$ and its gradient $\nabla_{\theta} f$ at $(x,\vartheta)$, cf. \cite[Section 1.1.2]{Nesterov2004}.) This oracle $\O$ can be construed as a black box ``description'' of the gradient of the loss function $f$ with respect to $\theta$. 

Given any particular iterative algorithm, $\A = \A(\D,\O)$, that solves our class of ERM problems, its \emph{first order oracle complexity} (or just \emph{oracle complexity}) $\Gamma(\A)$ is formally defined as the minimum number of oracle calls made by the algorithm $\A$ to obtain an $\epsilon$-approximate solution (in the sense of \eqref{Eq: Minimum of convex function}) \cite[Section 1.1]{Nesterov2004}. Thus, $\Gamma(\A)$ is typically a function of the notable problem variables $n,d,p,\eta,\epsilon$, and the fixed constants $L_1,L_2,\mu$. In this paper, following the paradigm of the optimization literature, we will focus on determining the first order oracle complexity of our proposed optimization method.

\section{Main Results and Discussion}
\label{Main Results and Discussion}

In this section, we first present our generalization of local polynomial regression to the multivariate case in \secref{Local Polynomial Regression}, and then describe our proposed algorithm and explain its convergence analysis in \secref{Inexact Gradient Descent Algorithm}.

\subsection{Local Polynomial Regression}
\label{Local Polynomial Regression}

We commence by introducing the essentials of \emph{local polynomial regression} specialized to our noiseless setting. Since most references on the subject only present the univariate (i.e., $d = 1$) framework, e.g., \cite{FanGijbels1996}, \cite[Section 1.6]{Tsybakov2009}, we rigorously carry out the much more involved calculations for the multivariate case by carefully generalizing the univariate development in \cite[Section 1.6]{Tsybakov2009}. We also refer readers to \cite[Section 2]{ClevelandLoader1996} for classical references and a brief history of smoothing using local regression.

Consider a function $g : [0,1]^d \rightarrow \R$ that belongs to the $(\eta,L_2)$-H\"{o}lder class, namely, $g$ is $l = \lceil \eta \rceil - 1$ times differentiable and satisfies \eqref{Eq: Holder condition} for every $s \in \Z_+^d$ with $|s| = l$. (\propref{Prop: Taylor Approximation of Holder Class Functions} in \appref{Taylor Approximation of Holder Class Functions} conveys some intuition behind what this condition means.) In order to present the local polynomial regression based approximation of this function, we first state some required definitions and notation. Let $K:[-1,1] \rightarrow [b,c]$ be a bounded \emph{kernel} function, where $c \geq b > 0$ are some fixed constants, and we will use the extension of this kernel (to the domain $\R$), $K:\R \rightarrow [0,c]$, in the sequel with the assumption that the support of $K$ is $[-1,1]$. For any fixed $m \in \N$, define the uniform grid
\begin{equation}
\label{Eq: Quantization points}
\G_m \triangleq \left\{u \in [0,1]^d : \forall i \in [d], \, u_i m \in [m]\right\} ,
\end{equation}
which has cardinality $|\G_m| = m^d$. Finally, define the vector(-valued function) $U:\R^d \rightarrow \R^{D}$ as:
\begin{equation}
\label{Eq: Def of U}
\forall u \in \R^d, \enspace U(u) \triangleq \left[\frac{u^s}{s!} : s \in \Z_+^d , \, |s| \leq l \right]^{\T} ,
\end{equation}
which has dimension $D \triangleq \dim(U) = \sum_{k = 0}^{l}{\binom{k + d - 1}{d-1}} = \binom{l+d}{d}$ (via the hockey-stick identity). Then, for any arbitrarily small value of the \emph{bandwidth} parameter $h \in \big(0,\frac{1}{2}\big)$, we can write the following weighted regression problem:
\begin{equation}
\label{Eq: Regression}
\forall x \in [h,1-h]^d, \enspace \hat{\Phi}(x) \triangleq \argmin{\phi \in \R^{D}} \sum_{y \in \G_m}{\left(g(y) - \phi^{\T} U\!\left(\frac{y-x}{h}\right)\right)^{\! 2} \prod_{j = 1}^{d}{K\!\left(\frac{y_j-x_j}{h}\right)}} \, ,
\end{equation}
where $\hat{\Phi}(x)  = \big[\hat{\Phi}_1(x) \, \cdots \, \hat{\Phi}_D(x)\big]^{\T}$. We refer to the first entry $\hat{\phi}(x) \triangleq \hat{\Phi}_1(x) = U(0)^{\T} \hat{\Phi}(x)$ of $\hat{\Phi}(x)$ as the \emph{local polynomial interpolator} of order $l$ for $g:[h,1-h]^d \rightarrow \R$ (which is $g:[0,1]^d \rightarrow \R$ restricted to the domain $[h,1-h]^d$). As we shall see, $\hat{\phi}$ provides a ``good'' approximation of the function $g$ in the supremum (or Chebyshev) norm.

For any $x \in [h,1-h]^d$, construct the $D \times D$ symmetric matrix
\begin{equation}
\label{Eq: B-matrix}
B(x) \triangleq \frac{1}{(m h)^d} \sum_{y \in \G_m}{U\!\left(\frac{y-x}{h}\right) U\!\left(\frac{y-x}{h}\right)^{\!\T} \prod_{j = 1}^{d}{K\!\left(\frac{y_j-x_j}{h}\right)}} \, ,
\end{equation}
which is also positive semidefinite due to the non-negativity of the kernel. We next present a key technical proposition which generalizes \cite[Lemma 1.5]{Tsybakov2009} to the $d$-variate setting and demonstrates that $B(x)$ is actually positive definite.

\begin{proposition}[L\"{o}wner Lower Bound]
\label{Prop: Lowner Lower Bound}
Suppose that $l \geq d \geq 19$. Then, for all $m \in \N$ and $h \in \big(0,\frac{1}{2}\big)$ such that $m h \geq 4 l (3e)^d/\Lambda(d,l)$ and all $x \in [h,1-h]^d$, we have
$$ B(x) \succeq \frac{b^d \Lambda(d,l)}{2} I_D \, , $$
where the constant $\Lambda(d,l) > 0$ is defined as
\begin{equation}
\label{Eq: Lower bound def}
\Lambda(d,l) \triangleq \left(\frac{\pi d}{8 e^2}\right)^{\! d \left(\frac{l+d}{d}\right)^{\! d}} \left(\left(\frac{l+d}{d}\right)^{\! d} - 1 \right)^{\left(\frac{l+d}{d}\right)^{\! d} - 1} \frac{1}{(l+d)^{3l \left(\frac{e(l+d)}{d}\right)^{\! d}}} \, .
\end{equation}
\end{proposition}

\propref{Prop: Lowner Lower Bound} is established in \secref{Proof of Prop Lowner Lower Bound} by exploiting several basic ideas from the theory of orthogonal polynomials and matrix perturbation theory. The proof is much more sophisticated than \cite[Lemma 1.5]{Tsybakov2009} because explicit quantitative estimates are derived to deal with general $d$. Since \propref{Prop: Lowner Lower Bound} shows that $B(x)$ is invertible, a straightforward exercise in calculus conveys that the \emph{unique} solution $\hat{\Phi}(x)$ of weighted regression problem in \eqref{Eq: Regression} can be written as:
\begin{equation}
\forall x \in [h,1-h]^d, \enspace \hat{\Phi}(x) = \sum_{y \in \G_m}{g(y) w_y(x)} \, ,
\end{equation}
where the optimal vector-valued weights $w_y(x) \in \R^{D}$ are given by:
\begin{equation}
\label{Eq: Vector Weights}
\forall y \in \G_m, \, \forall x \in [h,1-h]^d, \enspace w_y(x) = \frac{1}{(mh)^d} \left(\prod_{j = 1}^{d}{K\!\left(\frac{y_j-x_j}{h}\right)}\right) B(x)^{-1} U\!\left(\frac{y-x}{h}\right) .
\end{equation}
Then, defining the first coordinates of these vectors as the real-valued \emph{interpolation weights}:
\begin{equation}
\label{Eq: Weights}
\forall y \in \G_m, \, \forall x \in [h,1-h]^d, \enspace w_y^*(x) = \left[w_y(x)\right]_1 ,
\end{equation}
the local polynomial interpolator of order $l$ for $g:[h,1-h]^d \rightarrow \R$ can be written as the following weighted sum of the interpolation points $\{g(y) : y \in \G_m\}$:
\begin{equation}
\label{Eq: Local polynomial interpolator}
\forall x \in [h,1-h]^d, \enspace \hat{\phi}(x) = \sum_{y \in \G_m}{g(y) w_y^*(x)} \, .
\end{equation}
In \secref{Properties of Local Polynomial Interpolation Weights}, we present some important properties of the interpolation weights given in \eqref{Eq: Weights} based on the key estimate in \propref{Prop: Lowner Lower Bound}.

The ensuing theorem conveys that local polynomial regression can be used to uniformly approximate a function $g$ (in the supremum norm sense) for sufficiently large uniform grids $\G_m$.

\begin{theorem}[Supremum Norm Interpolation]
\label{Thm: Chebyshev Norm Interpolation}
Fix any constant $\delta \in (0,1)$, and consider the function $g:[0,1]^d \rightarrow \R$ which belongs to the $(\eta,L_2)$-H\"{o}lder class. Suppose that $l = \lceil \eta \rceil - 1 \geq d \geq 19$, and that the grid size $|\G_m| = m^d$ is sufficiently large so that 
$$ m \geq \frac{110 (2 L_2 + 1) c}{b} \left(\frac{d (3e)^d}{\Lambda(d,l)^{2}}\right) \! \left(\frac{1}{\delta}\right)^{\! 1/\eta} , $$
where $c \geq b > 0$ are the bounds on our kernel, and $\Lambda(d,l)$ is defined in \eqref{Eq: Lower bound def}. Then, the supremum norm approximation error of $g$ is bounded by $\delta$, viz.
$$ \sup_{x \in [h,1-h]^d}{\left|\hat{\phi}(x) - g(x)\right|} \leq \delta \, , $$
where $\hat{\phi}$ is the local polynomial interpolator of order $l$ for $g$ given in \eqref{Eq: Local polynomial interpolator}, and the bandwidth $h \in \big(0,\frac{1}{2}\big)$ is chosen to be 
$$ h = \frac{4 l (3e)^d}{m \Lambda(d,l)} \leq \frac{2 b}{55 (2 L_2 + 1) c} \left(\frac{\Lambda(d,l) l}{d}\right) . $$ 
\end{theorem}

\thmref{Thm: Chebyshev Norm Interpolation} is proved in \secref{Proof of Theorem Chebyshev Norm Interpolation} using the pertinent properties of the interpolation weights in \eqref{Eq: Weights} derived in \secref{Properties of Local Polynomial Interpolation Weights}. We note that the lower bound on $m$ is expected to scale at least exponentially with the dimension $d$ due to the curse of dimensionality, cf. \cite[Section 7.1]{FanGijbels1996}. However, it may be possible to improve the dependence we have above using a different analysis. In the sequel, we will exploit \thmref{Thm: Chebyshev Norm Interpolation} to perform inexact gradient descent.

\subsection{Inexact Gradient Descent using Local Polynomial Interpolation}
\label{Inexact Gradient Descent Algorithm}

We now turn to describing and analyzing a new inexact gradient descent algorithm that performs local polynomial interpolation at every iteration. In particular, we delineate the algorithm in \secref{Description of Algorithm}, and derive convergence rates in \secref{Analysis of Algorithm}.

\subsubsection{Description of Algorithm}
\label{Description of Algorithm}

Under the setup and assumptions described in \secref{Introduction}, suppose we seek to solve the optimization problem \eqref{Eq: ERM} in the $\epsilon$-approximate sense of \eqref{Eq: Minimum of convex function} based on some training data $\D$ with the parameter $h^{\prime} > 0$ given by
\begin{equation}
\label{Eq: Data bound}
h^{\prime} = \frac{2 b}{55 (2 L_2 + 1) c} \left(\frac{\Lambda(d,l) l}{d}\right) = O\!\left(\frac{\Lambda(d,l) l}{d}\right) ,
\end{equation}
which is the upper bound on the bandwidth $h$ in \thmref{Thm: Chebyshev Norm Interpolation}.\footnote{We remark that as long as $h^{\prime} = O(\Lambda(d,l) l/d)$, we can choose a kernel $K$ in the sequel such that \eqref{Eq: Data bound} holds.} We first present our proposed algorithm, dubbed \emph{local polynomial interpolation based gradient descent (LPI-GD)}, in general, and then specialize the values of different problem variables with judicious choices later on. 

Let $\theta^{(t)} \in \R^p$ denote the ``approximate solution'' produced by LPI-GD at iteration $t \in \Z_+$, and let $\delta_t \in (0,1)$ denote the maximum allowable supremum norm approximation error at iteration $t \in \N$. Fix any bounded kernel $K:[-1,1] \rightarrow [b,c]$ with $c \geq b > 0$ as in \secref{Local Polynomial Regression}, and corresponding to each $\delta_t$ at iteration $t \in \N$, consider the uniform grid $\G_{m_t}$, where $m_t \in \N$ is given by \thmref{Thm: Chebyshev Norm Interpolation}:
\begin{equation}
\label{Eq: Algo grid sizes}
m_t = \Bigg\lceil \frac{110 (2 L_2 + 1) c}{b} \left(\frac{d (3e)^d}{\Lambda(d,l)^{2}}\right) \! \left(\frac{1}{\delta_t}\right)^{\! 1/\eta} \Bigg\rceil \, ,
\end{equation}
and the bandwidth $h_t \in \big(0,\frac{1}{2}\big)$ is also given by \thmref{Thm: Chebyshev Norm Interpolation}:
\begin{equation}
\label{Eq: Special bandwidth selection}
h_t = \frac{4 l (3e)^d}{m_t \Lambda(d,l)} \leq h^{\prime} \, . 
\end{equation}
We initialize $\theta^{(0)} \in \R^p$ to be any arbitrary vector in $\R^p$ at iteration $t = 0$. Then, at iteration $t \in \N$, we make $|\G_{m_t}| = m_t^d$ first order oracle calls at all virtual data points in the uniform grid $\G_{m_t}$ with the approximate solution at iteration $t-1$:
\begin{equation}
\forall y \in \G_{m_t}, \enspace \O\big(y,\theta^{(t-1)}\big) = \nabla_{\theta} f(y;\theta^{(t-1)}) \, .  
\end{equation}
For every $i \in [p]$, these oracle calls give us information about the partial derivatives $\{g_i(y;\theta^{(t-1)}) : y \in \G_{m_t}\}$ with respect to $\theta_i$. Hence, akin to \eqref{Eq: Regression} and \eqref{Eq: Local polynomial interpolator} in \secref{Local Polynomial Regression}, we can construct the local polynomial interpolator $\hat{\phi}_{i}^{(t)} : [h^{\prime},1-h^{\prime}]^{d} \rightarrow \R$ of order $l = \lceil \eta \rceil - 1$ for the $(\eta,L_2)$-H\"{o}lder class function $g_i(\cdot;\theta^{(t-1)}) = \frac{\partial f}{\partial \theta_i}(\cdot;\theta^{(t-1)}):[h^{\prime},1-h^{\prime}]^d \rightarrow \R$ as follows: 
\begin{equation}
\label{Eq: Interpolators}
\forall x \in [h^{\prime},1-h^{\prime}]^d, \enspace \hat{\phi}_{i}^{(t)}(x) = \sum_{y \in \G_{m_t}}{g_i(y;\theta^{(t-1)}) w_{y,t}^*(x)} \, ,
\end{equation}
where the interpolation weights $w_{y,t}^*(x) \in \R$ are defined analogously to \eqref{Eq: Weights}; in particular, they are the first elements of the weight vectors $w_{y,t}(x) \in \R^D$, which are defined as in \eqref{Eq: Vector Weights}, but with $\G_m$ replaced by the grid $\G_{m_t}$ and $h$ replaced by the bandwidth $h_t$. Using the learned interpolators in \eqref{Eq: Interpolators} at iteration $t$, we next approximate the true gradient of the ERM objective function $\nabla_{\theta} F (\theta^{(t-1)})$ at $\theta^{(t-1)}$ as follows. For every $i \in [p]$ and every $j \in [n]$, we first estimate the partial derivative $g_i(x^{(j)};\theta^{(t-1)})$ at the data sample $x^{(j)}$ by computing $\hat{\phi}_{i}^{(t)}(x^{(j)})$. Then, since we have $\nabla_{\theta} F (\theta^{(t-1)}) = \frac{1}{n} \sum_{j = 1}^{n}{\nabla_{\theta} f (x^{(j)};\theta^{(t-1)})}$ due to \eqref{Eq: ERM Objective Function}, we can estimate $\nabla_{\theta} F (\theta^{(t-1)})$ by computing
\begin{equation}
\label{Eq: Approximation to Grad F}
\widehat{\nabla F}^{(t)} = \frac{1}{n} \sum_{j = 1}^{n}{\left[\hat{\phi}_{i}^{(t)}(x^{(j)}) : i \in [p]\right]^{\! \T}} ,
\end{equation}
where $\big[\hat{\phi}_{i}^{(t)}(x^{(j)}) : i \in [p]\big]^{\T} \in \R^{p}$ is our approximation of $\nabla_{\theta} f(x^{(j)};\theta^{(t-1)})$ at the data sample $x^{(j)}$. Finally, we perform the inexact gradient descent update
\begin{equation}
\label{Eq: Update step}
\theta^{(t)} = \theta^{(t-1)} - \frac{1}{L_1} \widehat{\nabla F}^{(t)} ,
\end{equation}
which produces the approximate solution for iteration $t$. We run the above procedure for $T \in \N$ iterations, and the LPI-GD algorithm returns $\theta^{(T)} \in \R^p$ as the $\epsilon$-approximate solution for the optimization problem \eqref{Eq: ERM}. The analysis in \secref{Analysis of Algorithm} elucidates the specific values taken by the constants $\delta_t$ and the total number of iterations $T$ in terms of the other problem parameters. In particular, we will see that the approximation errors $\delta_t$, grid sizes $m_t$, bandwidths $h_t$, and uniform grids $\G_{m_t}$ all do not change over iterations. A pseudocode summary of the LPI-GD algorithm with pertinent $\delta_t$ and $T$ (selected as in \secref{Analysis of Algorithm}) is presented in \algoref{Algorithm: LPI-GD}; the algorithm uses the training data $\D$ and a first order oracle $\O$ to generate an $\epsilon$-approximate solution that solves \eqref{Eq: ERM}. 

We remark that \algoref{Algorithm: LPI-GD} portrays a `theoretical' algorithm, because it assumes that problem parameters, such as the Lipschitz constant $L_1$, the strong convexity parameter $\mu$, the minimum value $F_*$, and the H\"{o}lder exponent $\eta$ and constant $L_2$, are known to the practitioner. When some or all of these parameters are unknown, we can use appropriate bounds and estimates based on the loss function $f : [0,1]^d \times \R^p \rightarrow \R$. For example, we may use any upper bound on $L_1$, any lower bound on $\mu$, and any universal lower bound on $f$ (which yields a lower bound on $F_*$). Moreover, for smooth functions $f$, we can choose any desired level of differentiability $\eta$ and any upper bound on $L_2$. When such bounds or estimates are also unavailable, the best we can do is to select $T$, $\delta_t$, $m_t$, and $h_t$ to have the correct scaling with respect to $\epsilon$, $d$, $p$, and $\eta$. It is worth mentioning that our objective in this paper is to demonstrate that the oracle complexities of GD and SGD can be theoretically beaten by the proposed algorithm. So, we do not delve further into considerations regarding its practical implementation here.

\begin{algorithm}[t]
\begin{algorithmic}[1]
\renewcommand{\algorithmicrequire}{\textbf{Input:}}
\Require $n$ samples of training data $\D = \{x^{(i)} \in [h^{\prime},1-h^{\prime}]^d : i \in [n]\}$
\Require first order oracle $\O : [0,1]^d \times \R^p \rightarrow \R^p$ defined in \eqref{Eq: Oracle}
\Require approximation accuracy level $\epsilon > 0$
\renewcommand{\algorithmicensure}{\textbf{Output:}}
\Ensure $\epsilon$-approximate solution $\theta^* \in \R^p$ satisfying \eqref{Eq: Minimum of convex function}
\Statex \textbf{\emph{Step 1: Initialization}}
\State Set arbitrary initial parameter vector $\theta^{(0)} \in \R^p$
\State Set number of iterations $T \in \N$ according to \eqref{Eq: Chosen number of iterations} 
\State Set supremum norm approximation error $\delta = (1 - \sigma^{-1})^{T/2}$ \Comment{This is the same for all iterations; see \eqref{Eq: Choices of the Chebyshev constants}}
\State Set grid size $m \in \N$ according to \eqref{Eq: Algo grid sizes} \Comment{Substitute $m_t = m$ and $\delta_t = \delta$ in \eqref{Eq: Algo grid sizes}}
\State Set bandwidth $h = 4 l (3e)^d / (m \Lambda(d,l))$  \Comment{See \eqref{Eq: Special bandwidth selection}}
\Statex \textbf{\emph{Step 2: Calculation of interpolation weights}}
\State Construct uniform grid $\G_m$ according to \eqref{Eq: Quantization points} 
\For{$j = 1$ to $n$} \Comment{Iterate over training data}
\State Compute interpolation weights $\{w_y^*(x^{(j)}) : y \in \G_m\}$ via \eqref{Eq: Vector Weights} and \eqref{Eq: Weights} using fixed kernel $K:[-1,1] \rightarrow [b,c]$
\EndFor
\Statex \textbf{\emph{Step 3: Inexact gradient descent}}
\For{$t = 1$ to $T$}
\State Make oracle calls $\{\O(y,\theta^{(t-1)}) : y \in \G_m\}$ to get partial derivatives $\big\{\frac{\partial f}{\partial \theta_i}(y;\theta^{(t-1)}) : y \in \G_m, \, i \in [p]\big\}$
\StatexInd{\textbf{\emph{Step 3a: Local polynomial interpolation}}}
\For{$j = 1$ to $n$} \Comment{Iterate over training data}
\State \parbox[t]{\dimexpr\linewidth-\algorithmicindent-\algorithmicindent}{Compute local polynomial interpolators $\big\{\hat{\phi}_i^{(t)}(x^{(j)}) : i \in [p]\big\}$ according to \eqref{Eq: Interpolators} using the oracle queries $\big\{\frac{\partial f}{\partial \theta_i}(y;\theta^{(t-1)}) : y \in \G_m, \, i \in [p]\big\}$ and precomputed interpolation weights $\{w_y^*(x^{(j)}) : y \in \G_m\}$\strut}
\EndFor
\State Construct estimate of gradient $\widehat{\nabla F}^{(t)}$ using interpolators $\big\{\hat{\phi}_i^{(t)}(x^{(j)}) : i \in [p], \, j \in [n]\big\}$ as shown in \eqref{Eq: Approximation to Grad F}
\StatexInd{\textbf{\emph{Step 3b: Standard GD update}}}
\State Update: $\theta^{(t)} \leftarrow \theta^{(t-1)} - \frac{1}{L_1} \widehat{\nabla F}^{(t)}$ \Comment{See \eqref{Eq: Update step}}
\EndFor
\State \Return $\theta^* = \theta^{(T)}$
\end{algorithmic}
\caption{Local polynomial interpolation based gradient descent (LPI-GD) algorithm to approximately solve \eqref{Eq: ERM}.}
\label{Algorithm: LPI-GD}
\end{algorithm}

Before proceeding to our analysis, it is worth juxtaposing \algoref{Algorithm: LPI-GD} with standard batch GD and SGD. At each iteration, GD exactly computes gradients for all data samples using $n$ oracle calls. Since this can be computationally expensive, inexact gradient descent methods attempt to reduce the number of oracle calls per iteration at the cost of increasing the number of iterations. For example, SGD makes only one oracle call per iteration for a randomly chosen data sample, and uses this queried gradient value as a proxy for the true gradient. In contrast, at each iteration, our approach evaluates the gradients at certain representative virtual points in the data space, and uses these oracle calls to learn a model for the gradient oracle. This learnt model is then used to approximate the true gradient. Intuitively, when the loss function is a sufficiently smooth function of the data, we can learn an accurate model with a reasonably small number of virtual points. 

\subsubsection{Analysis of Algorithm}
\label{Analysis of Algorithm}

To compute the oracle complexity of the LPI-GD algorithm, we must first compute the minimum number of iterations it takes to obtain an $\epsilon$-approximate solution to \eqref{Eq: ERM}. The ensuing proposition chooses appropriate values of $\delta_t$, which turn out to be constant with respect to $t$, and develops an upper bound on the minimum number of iterations $T$.

\begin{proposition}[Iteration Complexity]
\label{Prop: Number of Iterations}
Suppose that $l = \lceil \eta \rceil - 1 \geq d \geq 19$ and the assumptions of \secref{Smoothness Assumptions and Approximate Solutions} hold. For any (small) accuracy $\epsilon > 0$, define the number of iterations $T \in \N$ as
\begin{equation}
\label{Eq: Chosen number of iterations}
T = \Bigg\lceil \left(\log\!\left(\frac{\sigma}{\sigma - 1}\right)\right)^{\! -1} \log\!\left(\frac{F(\theta^{(0)}) - F_* + \frac{p}{2 \mu}}{\epsilon}\right) \Bigg\rceil \, , 
\end{equation}
and the maximum allowable supremum norm approximation errors $\{\delta_t \in (0,1) : t \in [T]\}$ via:
\begin{equation}
\label{Eq: Choices of the Chebyshev constants}
\forall t \in [T], \enspace \delta_t^2 = \left(1 - \frac{1}{\sigma}\right)^{\! T} = \Theta\!\left(\frac{\epsilon}{p}\right) ,
\end{equation}
where $F : \R^p \rightarrow \R$ is the ERM objective function in \eqref{Eq: ERM Objective Function}, and $\sigma = L_1/\mu > 1$ is the condition number defined in \secref{Smoothness Assumptions and Approximate Solutions}. Then, after $T$ iterations of the LPI-GD algorithm described in \algoref{Algorithm: LPI-GD} in \secref{Description of Algorithm} with arbitrary initialization $\theta^{(0)} \in \R^p$ and updates \eqref{Eq: Update step}, the parameter vector $\theta^{(T)} \in \R^p$ is an $\epsilon$-approximate solution to the optimization problem in \eqref{Eq: ERM} in the sense of \eqref{Eq: Minimum of convex function}:
$$ F(\theta^{(T)}) - F_* \leq \epsilon \, . $$
\end{proposition}

\propref{Prop: Number of Iterations} is proved in \secref{Proof of Prop Number of Iterations} using \thmref{Thm: Chebyshev Norm Interpolation} and some known ideas from the analysis of inexact gradient descent methods in \cite{FriedlanderSchmidt2012}. Using \propref{Prop: Number of Iterations}, the ensuing theorem bounds the first order oracle complexity of the LPI-GD algorithm in \algoref{Algorithm: LPI-GD}, denoted $\Gamma(\text{LPI-GD})$, which is the minimum number of oracle calls made by \algoref{Algorithm: LPI-GD} to obtain an $\epsilon$-approximate solution (as defined in \secref{Oracle Complexity}). 

\begin{theorem}[Oracle Complexity]
\label{Thm: Oracle Complexity}
Suppose that $l = \lceil \eta \rceil - 1 \geq d \geq 19$ and the assumptions of \secref{Smoothness Assumptions and Approximate Solutions} hold. For any (small) accuracy $0 < \epsilon \leq \frac{(L_1 - \mu)p}{2 L_1 \mu}$, consider the LPI-GD algorithm described in \algoref{Algorithm: LPI-GD} in \secref{Description of Algorithm} with arbitrary initialization $\theta^{(0)} \in \R^p$, number of iterations $T \in \N$ as in \eqref{Eq: Chosen number of iterations}, maximum allowable approximation errors $\{\delta_t \in (0,1) : t \in [T]\}$ as in \eqref{Eq: Choices of the Chebyshev constants}, and updates \eqref{Eq: Update step}. Then, the first order oracle complexity of this algorithm (to obtain an $\epsilon$-approximate solution in the sense of \eqref{Eq: Minimum of convex function}) is upper bounded by
$$ \Gamma(\text{\emph{LPI-GD}}) \leq C_{\mu,L_1,L_2,b,c}(d,l) \left(\frac{p + 2 \mu (F(\theta^{(0)}) - F_*)}{\epsilon}\right)^{\! d/(2\eta)} \log\!\left(\frac{p + 2 \mu (F(\theta^{(0)}) - F_*)}{2 \mu \epsilon}\right) $$
where we let 
\begin{equation}
\label{Eq: Scaling in Main Theorem}
C_{\mu,L_1,L_2,b,c}(d,l) \triangleq \left(\frac{\sigma + 2(L_1 - \mu)}{(L_1 - \mu)\log\!\left(\frac{\sigma}{\sigma - 1}\right)}\right) \! \left( \frac{220 (2 L_2 + 1) c}{b} \right)^{\! d} \! \left(\frac{d^d (3e)^{d^2}}{\Lambda(d,l)^{2 d}}\right) ,
\end{equation}
where $\Lambda(d,l)$ is defined in \eqref{Eq: Lower bound def}, and $\mu,L_1,L_2,b,c$ are constants defined in Sections \ref{Smoothness Assumptions and Approximate Solutions} and \ref{Local Polynomial Regression}.
\end{theorem}

\thmref{Thm: Oracle Complexity} is established in \secref{Proof of Thm Oracle Complexity}. Unlike the dimension-free oracle complexity bounds for batch GD and SGD, our bound in \thmref{Thm: Oracle Complexity} depends on $d$ and $p$. However, even though it may not be immediately obvious, the oracle complexity of LPI-GD scales much better than the oracle complexities of GD or SGD in many important regimes. We next elucidate a scaling law for the bound in \thmref{Thm: Oracle Complexity} in the setting where $n,d,p,\eta,\epsilon^{-1} \rightarrow \infty$ and all other problem parameters are constant. 

For comparison, recall that the first order oracle complexity of GD (with appropriately chosen constant step size) is bounded by \cite[Theorem 2.1.15]{Nesterov2004}:
\begin{equation}
\label{Eq: Gamma of GD}
\Gamma(\text{GD}) \leq \Gamma_*(\text{GD}) \triangleq \Theta\!\left(n \log\!\left(\frac{1}{\epsilon}\right)\right) , 
\end{equation}
when the ERM objective function in \eqref{Eq: ERM Objective Function} is strongly convex and has Lipschitz continuous gradient. (These conditions follow from the assumptions in \secref{Smoothness Assumptions and Approximate Solutions} as shown in the proof of \propref{Prop: Number of Iterations} in \secref{Proof of Prop Number of Iterations}. Note that each iteration of GD makes $n$ first order oracle queries, and \cite[Theorem 2.1.15]{Nesterov2004} conveys that it takes $O(\log(\epsilon^{-1}))$ iterations for GD to produce an $\epsilon$-approximate solution in the sense of \eqref{Eq: Minimum of convex function}.) The expression in \eqref{Eq: Gamma of GD} does not depend on $d,p,\eta$, and neglects all other constant parameters. Similarly, when the ERM objective function in \eqref{Eq: ERM Objective Function} is strongly convex, has Lipschitz continuous gradient, and certain universal second moment bounds are satisfied, the first order oracle complexity of SGD (with linearly diminishing step sizes) is bounded by \cite{Nemirovskietal2009}:
\begin{equation}
\label{Eq: Gamma of SGD}
\Gamma(\text{SGD}) \leq \Gamma_*(\text{SGD}) \triangleq \Theta\!\left(\frac{1}{\epsilon}\right) , 
\end{equation}
where the notion of $\epsilon$-approximate solution in \eqref{Eq: Minimum of convex function} is defined in expectation because the iterates of SGD are random variables. 

For simplicity, we will assume that the quantities $d,p,\eta,\epsilon^{-1}$ are all functions of the number of training samples $n$, and study the asymptotics as $n \rightarrow \infty$. The ensuing proposition presents a scaling regime for LPI-GD which beats the oracle complexities (or more precisely, the oracle complexity bounds) of both GD and SGD shown above.

\begin{proposition}[Comparison to GD and SGD]
\label{Prop: Comparison to GD and SGD}
Fix any constants $\alpha > 0$, $\beta > 0$, $\tau > \max\{1,\alpha^{-1}\}$, and $\gamma > \max\{1,\tau(\alpha + \beta)/2\}$. Suppose that $\epsilon = \Theta(n^{-\alpha})$, $p = O(n^{\beta})$, $d \leq \frac{\log\log(n)}{4 \log(e(\gamma + 1))}$, and $\eta = \Theta(d)$ such that $\eta \geq \tau (\alpha + \beta) d / 2$ and $d \leq l = \lceil \eta \rceil - 1 \leq \gamma d$. Furthermore, suppose that $\mu,L_1,L_2,b,c$ and $F(\theta^{(0)}) - F_*$ are $\Theta(1)$ with respect to $n$. Then, we have
$$ \Gamma(\text{\emph{LPI-GD}}) = O\!\left(\log(n) \exp\!\left(2 \sqrt{\log(n)}\right) n^{\! 1/\tau}\right) . $$
This implies that 
$$ \lim_{n \rightarrow \infty}{\frac{\Gamma(\text{\emph{LPI-GD}})}{\Gamma_*(\text{\emph{GD}})}} = \lim_{n \rightarrow \infty}{\frac{\Gamma(\text{\emph{LPI-GD}})}{\Gamma_*(\text{\emph{SGD}})}} = 0 \, . $$
\end{proposition}

\propref{Prop: Comparison to GD and SGD} is proved in \secref{Proof of Prop Comparison to GD and SGD}. Under its assumptions, the LPI-GD algorithm beats any $\Theta(\poly(n))$ oracle complexity when $\tau$ is chosen to be large enough. Hence, one can also calculate interesting regimes where LPI-GD scales better than more refined gradient-based algorithms such as accelerated GD \cite{Nesterov1983}, mini-batch SGD \cite{Dekeletal2012}, variance reduced SGD \cite{JohnsonZhang2013}, etc. We do not explicitly carry out these calculations for brevity.

We close this section by making several remarks. Firstly, many of the refined gradient-based methods mentioned above lead to improvements in the dependence of oracle complexity bounds on the condition number. We do not consider or optimize the trade-off with condition number in our analysis; this is in many ways an orthogonal question. As indicated at the outset of this paper, our focus has been on beating GD and SGD's oracle complexities with regard to their dependence on $n$, $\epsilon$, $d$, and $p$.

Secondly, while methods like SGD or mini-batch SGD improve upon GD's oracle complexity (in terms of its trade-off between $n$ and $\epsilon$), they are stochastic algorithms that only generate $\epsilon$-approximate solutions on expectation. In contrast, the LPI-GD algorithm can beat GD and is completely deterministic.

Thirdly, since the LPI-GD algorithm has to learn the oracle at every iteration, the computational (or ``arithmetical'') complexity of its iterations may seem to be higher than usual iterative methods. However, the motivation for focusing on the concept of oracle complexity in optimization theory is partly that computing gradients can be extremely hard for certain functions. So, the computational complexity of the oracle usually dominates any other reasonable computation done during iterations. We, therefore, stick to this canonical paradigm in this work \cite{Nesterov2004}. If the precise running time of the first order oracle $\O$ is known, then the complete computational complexity of the LPI-GD algorithm can be derived from \propref{Prop: Number of Iterations} and \thmref{Thm: Oracle Complexity} by accounting for the running time of the $D \times D$ matrix inversions performed at every iteration.

Fourthly, we emphasize that \propref{Prop: Comparison to GD and SGD} holds in the important regime where $\alpha \leq 1$ and $\beta > 1$. The former condition ensures that SGD is favorable to GD (since $\lim_{n \rightarrow \infty}{\Gamma_*(\text{SGD})/\Gamma_*(\text{GD})} = 0$), which is often the case in practice. The latter condition ensures that the number of parameters $p$ can be much large than the number of training samples $n$, which corresponds to the popular \emph{overparametrized regime} considered in the context of deep neural networks, cf. \cite{PoggioBanburskiLiao2020}.

Finally, we note that the dependence of the bounds in \thmref{Thm: Oracle Complexity} and \propref{Prop: Comparison to GD and SGD} on both $p$ and $d$ could potentially be improved. The dependence on $p$ arises because \thmref{Thm: Chebyshev Norm Interpolation} holds for functions with codomain $\R$. However, the proof of \propref{Prop: Number of Iterations} in \secref{Proof of Prop Number of Iterations} needs to bound the $\ell^2$-norm error between \eqref{Eq: Approximation to Grad F} and the true gradient $\nabla_{\theta} F (\theta^{(t-1)})$ for $t \in \N$. Translating the coordinate-wise guarantee to an $\ell^2$-norm guarantee in this proof introduces the dependence on $p$ (see \eqref{Eq: p dependence}). If some version of \thmref{Thm: Chebyshev Norm Interpolation} could be proved for functions with codomain $\R^p$, then we could possibly improve the dependence of \thmref{Thm: Oracle Complexity} on $p$. 

On the other hand, the dependence on $d$ could be improved by tightening the minimum eigenvalue lower bound in \lemref{Lemma: Auxiliary Loewner Lower Bound} as explained in \secref{Proof of Prop Lowner Lower Bound}. Nevertheless, this eigenvalue lower bound must decay at least exponentially (and correspondingly, the grid size parameter $m$ must grow at least exponentially) in $d$, as noted after \thmref{Thm: Chebyshev Norm Interpolation} in \secref{Local Polynomial Regression}. Thus, it is intuitively straightforward to see that the best scaling of $d$ we can hope for, so that the conclusion of \propref{Prop: Comparison to GD and SGD} remains true, is $d = O(\log(n))$. This logarithmic scaling of data dimension $d$ with number of training samples $n$ is significant. Indeed, for many data science applications of interest, the useful information in a dataset of $n$ high-dimensional feature vectors is contained within the pairwise (Euclidean) distances between different feature vectors. In these scenarios, standard low-dimensional approximate isometric embedding theorems, e.g., the \emph{Johnson-Lindenstrauss lemma} (cf. \cite{DasguptaGupta2003} and references therein), ensure that the feature vectors can be mapped to a $d = O(\log(n))$ dimensional space such that the pairwise distances are preserved. We leave the development of our results along these directions as future work.

\section{Proofs for Local Polynomial Interpolation}
\label{Proofs for Local Polynomial Interpolation} 

In this section, we provide proofs for the results presented in \secref{Local Polynomial Regression}. Specifically, we first establish \propref{Prop: Lowner Lower Bound}, then portray some properties of the interpolation weights in \eqref{Eq: Weights}, and finally prove \thmref{Thm: Chebyshev Norm Interpolation}.

\subsection{Proof of \propref{Prop: Lowner Lower Bound}}
\label{Proof of Prop Lowner Lower Bound}

To prove \propref{Prop: Lowner Lower Bound}, we require the following intermediate lemma which generalizes \cite[Lemma 1.4]{Tsybakov2009} to our $d$-variate setting.

\begin{lemma}[Auxiliary L\"{o}wner Lower Bound]
\label{Lemma: Auxiliary Loewner Lower Bound}
When $l \geq d \geq 19$, the $D \times D$ symmetric matrix
$$ \mathcal{B} \triangleq \int_{[-1,1]^d}{U(u) U(u)^{\T} \diff{u}} $$
satisfies the L\"{o}wner lower bound
$$ \B \succeq \Lambda(d,l) I_D \, , $$
where $U : \R^d \rightarrow \R^D$ is defined in \eqref{Eq: Def of U}, and the constant $\Lambda(d,l) > 0$ is defined in \eqref{Eq: Lower bound def}.
\end{lemma}

\begin{proof}
We begin by establishing some simple facts about $\B$. Following the proof of \cite[Lemma 1.4]{Tsybakov2009} mutatis mutandis, note that $\B$ is positive semidefinite:
$$ \forall v \in \R^{D}, \enspace v^{\T} \B v = \int_{[-1,1]^d}{\left(v^{\T} U(u)\right)^2 \diff{u}} \geq 0 \, . $$ 
Suppose for the sake of contradiction that there exists a non-zero vector $v \in \R^{D}$ such that 
$$ v^{\T} \B v = \int_{[-1,1]^d}{\left(v^{\T} U(u)\right)^2 \diff{u}} = 0 \, . $$ 
Then, we must have $v^{\T} U(u)= 0$ almost everywhere (with respect to the Lebesgue measure) on $[-1,1]^d$. However, since the non-zero polynomial $u \mapsto v^{\T} U(u)$ is analytic, its set of roots has Lebesgue measure zero (see, e.g., \cite[Section 4.1]{KrantzParks2002}). This is a contradiction. So, we actually have $v^{\T} \B v > 0$ for all non-zero vectors $v \in \R^{D}$. Thus, $\B$ is positive definite. From hereon, our proof deviates greatly from that of \cite[Lemma 1.4]{Tsybakov2009}.

For convenience in the sequel, we compute the entries of $\B$. Since the rows and columns of $\B$ are indexed by $\{s \in \Z_+^d : |s| \leq l\}$, for any $r,s \in \Z_+^d$ with $|r|,|s| \leq l$, we have
\begin{align}
\B_{r,s} & = \int_{[-1,1]^d}{\frac{u^{r+s}}{r! s!} \diff{u}} \nonumber \\
& = \frac{2^d}{r! s!} \int_{[-1,1]^d}{\frac{1}{2^d} \prod_{j = 1}^{d}{u_j^{r_j + s_j}} \diff{u}} \nonumber \\
& = \frac{2^d}{r! s!} \prod_{j = 1}^{d}{\E\!\left[X_j^{r_j + s_j} \right]} \nonumber \\
& = \I\!\left\{r + s \text{ is even entry-wise}\right\} \frac{2^d}{r! s!} \prod_{j = 1}^{d}{\frac{1}{r_j + s_j + 1}} 
\label{Eq: Entries of B}
\end{align}
where $X_1,\dots,X_d \in [-1,1]$ are independent and identically distributed (i.i.d.) uniform random variables, and the $k$th moment of $X_1$ is $\E[X^k] = 0$ for all odd $k \in \Z_+$ and $\E[X^k] = (k+1)^{-1}$ for all even $k \in \Z_+$, cf. \cite[Chapter 40]{Forbesetal2011}. Thus, $\B$ is a doubly non-negative matrix with all rational entries. 

We next lower bound the strictly positive minimum eigenvalue of $\B$. To this end, we first upper bound the trace of $\B$, and then lower bound the determinant of $\B$. To upper bound $\tr(\B)$, observe that
\begin{align}
\tr(\B) & = \sum_{s \in \Z_+^d : \, |s| \leq l}{\B_{s,s}} \nonumber \\
& = 2^d \sum_{s \in \Z_+^d : \, |s| \leq l}{\frac{1}{(s!)^2} \prod_{j = 1}^{d}{\frac{1}{2 s_j + 1}}} \nonumber \\
& \leq 2^d \sum_{k = 0}^{l}{\frac{1}{k!} \sum_{s \in \Z_+^d : \, |s| = k}{\frac{k!}{s!}}} \nonumber \\
& = 2^d \sum_{k = 0}^{l}{\frac{d^k}{k!}} \nonumber \\
& \leq (2e)^d \, ,
\label{Eq: Trace upper bound}
\end{align} 
where the second equality uses \eqref{Eq: Entries of B}, the fourth equality follows from the multinomial theorem, and the fifth inequality uses the Maclaurin series of the exponential function. (Note that the exponential scaling of $\tr(\B)$ is tight, because $\tr(\B) \geq \B_{(0,\dots,0),(0,\dots,0)} = 2^d$.)

To lower bound $\det(\B)$, we will first compute it exactly. To do this, we introduce the \emph{Legendre polynomials}, $\{L_k : [-1,1] \rightarrow \R \, | \, k \in \Z_+\}$, which are defined by Rodrigues' formula \cite[Chapter V, Equation (2.1)]{Chihara2011}:
\begin{equation}
\forall t \in [-1,1], \enspace L_k(t) \triangleq \frac{\sqrt{k + \frac{1}{2}}}{2^k k!} \frac{\diff{}^k}{\diff{t}^k} (t^2 - 1)^k \, , 
\end{equation}
where $L_k$ is a polynomial with degree $k$.\footnote{Alternatively, the Legendre polynomials can be obtained by applying the Gram-Schmidt process to the monomials $1,t,t^2,\dots$ with respect to the standard inner product, using a three-term recurrence relation, solving Legendre's differential equation, or expanding certain generating functions \cite{Chihara2011}.} It is well-known that the family of Legendre polynomials satisfies the orthonormality property \cite[Chapter I, Exercise 1.6]{Chihara2011}:
\begin{equation}
\label{Eq: Univariate Orthogonality}
\forall j,k \in \Z_+, \enspace \int_{-1}^{1}{L_j(t) L_k(t) \diff{t}} = \I\{j = k\} \, . 
\end{equation}
Moreover, for each $k \in \Z_+$, the leading coefficient (of $t^k$) of $L_k(t)$ is known to be \cite[Chapter V, Equation (2.7)]{Chihara2011}: 
\begin{equation}
\label{Eq: Univariate Leading Coeff}
\frac{\sqrt{k + \frac{1}{2}}}{2^k} \binom{2k}{k} \, . 
\end{equation}
Next, for any $r \in \Z_+^d$, define the \emph{Legendre product polynomial}:
\begin{equation}
\forall u \in [-1,1]^d, \enspace L_{r}(u) \triangleq \prod_{j = 1}^{d}{L_{r_j}(u_j)} 
\end{equation}
which has degree $r$ (with abuse of notation). It is straightforward to verify that these product polynomials also satisfy an orthonormality condition:
\begin{equation}
\label{Eq: Multivariate Orthogonality}
\forall r,s \in \Z_+^d, \enspace \int_{[-1,1]^d}{L_r(u) L_s(u) \diff{u}} = \prod_{j = 1}^{d}{\int_{-1}^{1}{L_{r_j}(u_j)L_{s_j}(u_j) \diff{u_j}}} = \I\{r = s\} \, , 
\end{equation}
where we utilize \eqref{Eq: Univariate Orthogonality}. Moreover, for each $r \in \Z_+^d$, the leading coefficient (of $u^r$) of $L_r(u)$ is known to be
\begin{equation}
\label{Eq: Multivariate Leading Coeff}
\prod_{j = 1}^{d}{\frac{\sqrt{r_j + \frac{1}{2}}}{2^{r_j}} \binom{2 r_j}{r_j}} = \frac{1}{2^{|r|}} \prod_{j = 1}^{d}{\sqrt{r_j + \frac{1}{2}} \binom{2 r_j}{r_j}} \, , 
\end{equation}
where we utilize \eqref{Eq: Univariate Leading Coeff}. Due to \eqref{Eq: Multivariate Orthogonality}, we know that for any $r \in \Z_+^d$, the set $\{L_s : s \in \Z_+^d, \, s \leq r\}$ is linearly independent and spans the set of all $d$-variate polynomials with degree at most $r$. (Note that we use the lexicographical order over indices in $\Z_+^d$ here. Furthermore, it is straightforward to verify that all $d$-variate monomials with non-zero coefficients that appear in $L_s$ have degree at most $s$.) As a result, defining the vector(-valued) function $\L : \R^d \rightarrow \R^{D}$ as:
\begin{equation}
\forall u \in [-1,1]^d, \enspace \L(u) \triangleq \left[L_s(u) : s \in \Z_+^d, \, |s| \leq l \right]^{\T} , 
\end{equation} 
which is indexed in lexicographical order, we have the linear relation:
\begin{equation}
\label{Eq: QR relation}
\forall u \in [-1,1]^d, \enspace U(u) = R \L(u) 
\end{equation}
for some fixed matrix $R \in \R^{D \times D}$ (that does not depend on $u$). It is evident that this matrix $R$ is lower triangular due to the lexicographical order we have imposed. Furthermore, it is easy to derive the diagonal entries of $R$. Indeed, for any index $s \in \Z_+^d$ with $|s| \leq l$, equating the leading coefficients on both sides of \eqref{Eq: QR relation} produces 
$$ \frac{1}{s!} = R_{s,s} \frac{1}{2^{|s|}} \prod_{j = 1}^{d}{\sqrt{\frac{2 s_j + 1}{2}} \binom{2 s_j}{s_j}} $$
using \eqref{Eq: Multivariate Leading Coeff}. Rearranging this, we get:
\begin{equation}
\label{Eq: Diagonal Cholesky}
\forall s \in \Z_+^d \text{ with } |s| \leq l, \enspace R_{s,s} = \frac{2^{|s|} 2^{d/2}}{s!} \prod_{j = 1}^{d}{\frac{1}{ \sqrt{2 s_j + 1}} \binom{2 s_j}{s_j}^{\! -1}} \, . 
\end{equation}
Now notice that
\begin{align}
\mathcal{B} & = \int_{[-1,1]^d}{U(u) U(u)^{\T} \diff{u}} \nonumber \\
& = R \left(\int_{[-1,1]^d}{\L(u) \L(u)^{\T} \diff{u}} \right) R^{\T} \nonumber \\
& = R I_D R^{\T} \nonumber \\
& = R R^{\T}
\label{Eq: Cholesky decomp}
\end{align}
where the second equality follows from \eqref{Eq: QR relation}, and the third equality follows from \eqref{Eq: Multivariate Orthogonality}. We recognize \eqref{Eq: Cholesky decomp} as the \emph{Cholesky decomposition} of $\B$ \cite[Corollary 7.2.9]{HornJohnson2013}. Thus, using \eqref{Eq: Cholesky decomp}, we can write
\begin{align}
\det(\B) & = \prod_{s \in \Z_+^d: \, |s| \leq l}{R_{s,s}^2} \nonumber \\
& = \prod_{s \in \Z_+^d: \, |s| \leq l}{ \underbrace{4^{|s|} \prod_{i = 1}^{d}{\binom{2 s_i}{s_i}^{\! -2}}}_{\leq \, 1} \cdot \underbrace{\frac{2^{d}}{(s!)^2} \prod_{j = 1}^{d}{\frac{1}{2 s_j + 1}}}_{= \, \B_{s,s}}} \nonumber \\
& \geq \prod_{s \in \Z_+^d: \, |s| \leq l}{\frac{(\pi/2)^{d}}{4^{|s|} (s!)^2} \underbrace{\prod_{j = 1}^{d}{\frac{4s_j + 1}{2 s_j + 1}}}_{\geq \, 1}} \nonumber \\
& \geq \left(\frac{\pi}{2}\right)^{\! d D} \left(\prod_{s \in \Z_+^d: \, |s| \leq l}{\frac{1}{2^{|s|} s!}}\right)^{\! 2} \nonumber \\
& = \left(\frac{\pi}{2}\right)^{\! d D} \frac{1}{4^{\sum_{i = 0}^{l}{i \binom{i + d - 1}{d - 1}}}} \left(\prod_{j = 0}^{l}{\prod_{s \in \Z_+^d: \, |s| = j}{\frac{1}{s_1! \cdots s_d!}}}\right)^{\! 2} \nonumber \\
& \geq \left(\frac{\pi}{2 e^2}\right)^{\! d D} \left(\frac{e^2}{4}\right)^{\! \sum_{i = 0}^{l}{i \binom{i + d - 1}{d - 1}}}\prod_{j = 0}^{l}{\prod_{s \in \Z_+^d: \, |s| = j}{\frac{1}{(s_1 + 1)^{2s_1 + 1} \cdots (s_d + 1)^{2s_d + 1}}}} \nonumber \\
& \geq \left(\frac{\pi}{4 e}\right)^{\! d D} \prod_{j = 0}^{l}{\prod_{s \in \Z_+^d: \, |s| = j}{\frac{1}{(s_1 + 1)^{2s_1 + 1} \cdots (s_d + 1)^{2s_d + 1}}}} \nonumber \\
& \geq \left(\frac{\pi}{4 e}\right)^{\! d D} \prod_{j = 0}^{l}{\left(\frac{d^d}{(j+d)^d (j+1)^{2 j}}\right)^{\!\binom{j + d - 1}{d - 1}}} \nonumber \\
& \geq \left(\frac{\pi d}{4 e}\right)^{\! d D} \frac{1}{(l+d)^{\sum_{j = 0}^{l}{(2 j + d)\binom{j + d - 1}{d - 1}}}} \nonumber \\
& \geq \left(\frac{\pi d}{4 e}\right)^{\! d D} \frac{1}{(l+d)^{3lD}} \, ,
\label{Eq: Det lower bound}
\end{align}
where the second equality follows from \eqref{Eq: Diagonal Cholesky}, the third inequality follows from the following bound (on Catalan numbers; see \cite[Theorem]{DuttonBrigham1986}):
$$ \forall j \in \Z_+, \enspace \binom{2j}{j}^{\! 2} < \frac{4^{2j+1}}{\pi(4j + 1)} \, , $$
the fifth equality follows from adding the exponents $|s|$ of $2$ over all $s \in \Z_+^d$ with $|s| \leq l$ and then squaring, the sixth inequality follows from the Stirling's formula bound (see, e.g., \cite[Chapter II, Section 9, Equation (9.15)]{Feller1968}):\footnote{Note that the usual upper bound of Stirling's approximation is $j! \leq e \sqrt{j} (j/e)^{j}$ for all $j \in \N$. In order to include the case $j = 0$, we loosen this bound in our analysis.} 
$$ \forall j \in \Z_+, \enspace j! \leq (j+1)^{j+\frac{1}{2}} e^{-j+1} \, , $$
the seventh inequality follows from the straightforward bound:
$$ \sum_{i = 0}^{l}{i \binom{i + d - 1}{d - 1}} \geq l \binom{l + d - 1}{d - 1} = \frac{d l}{l+d} \binom{l + d}{d} \geq \frac{d}{2} \binom{l + d}{d} = \frac{d D}{2} $$
which uses the fact that $l \geq d$, the eighth inequality uses the following basic facts and calculations:
\begin{enumerate}
\item The hypercube maximizes the volume over all $d$-orthotopes with an isoperimetric constraint:\footnote{This is a simple consequence of the arithmetic mean--geometric mean (AM--GM) inequality.}
$$ \max_{\substack{s_1,\dots,s_d \geq 0:\\s_1 + \cdots + s_d = j}}{(s_1+1) \cdots (s_d+1)} = \left(\frac{j}{d} + 1\right)^{\! d} $$
where $j \in \Z_+$ and the maximum is achieved by $s_1 = \cdots = s_d = \frac{j}{d}$;  
\item For any $j \in \Z_+$ and any $s \in \Z_+^d$ with $|s| = j$, the vector $(j,0,\dots,0) \in \Z_+^d$ majorizes the vector $s$, and since the map $0 \leq t \mapsto 2 t \log(t + 1)$ is convex, Karamata's majorization inequality yields (cf. \cite[Chapter 1, Section A]{MarshallOlkinArnold2011})
$$ \sum_{i = 1}^{d}{2 s_i \log(s_i + 1)} \leq 2 j \log(j + 1) \, , $$
which implies that
$$ \prod_{i = 1}^{d}{(s_i + 1)^{2 s_i}} \leq (j+1)^{2j} $$
using the monotonicity of the exponential function;
\end{enumerate} 
and the tenth inequality follows from the simple bound:
$$ \sum_{j = 0}^{l}{(2j + d) \binom{j + d - 1}{d - 1}} \leq (2 l + d) \sum_{j = 0}^{l}{\binom{j + d - 1}{d - 1}} = (2 l + d)D \leq 3 l D $$
which uses the fact that $l \geq d$. 

Finally, recall the following well-known result, cf. \cite[Lemma 1]{MerikoskiVirtanen1997}:
\begin{equation}
\label{Eq: Eigenvalue lower bound}
\lambda_{\mathsf{min}}(\B) \geq \det(\B) \left(\frac{D-1}{\tr(B)}\right)^{\! D-1} \, ,
\end{equation}
which utilizes the positive definiteness of $\B$ established earlier.\footnote{This bound can also be proved using the AM--GM inequality.} Combining \eqref{Eq: Trace upper bound}, \eqref{Eq: Det lower bound}, and \eqref{Eq: Eigenvalue lower bound}, we obtain
\begin{align}
\lambda_{\mathsf{min}}(\B) & \geq \left(\frac{\pi d}{4 e}\right)^{\! d D} \frac{(D-1)^{D-1}}{(l+d)^{3lD} (2e)^{d(D-1)}} \nonumber \\
& \geq \left(\frac{\pi d}{8 e^2}\right)^{\! d D} \frac{(D-1)^{D-1}}{(l+d)^{3lD}} \nonumber \\
& \geq \left(\frac{\pi d}{8 e^2}\right)^{\! d \left(\frac{l+d}{d}\right)^{\! d}} \left(\left(\frac{l+d}{d}\right)^{\! d} - 1 \right)^{\left(\frac{l+d}{d}\right)^{\! d} - 1} \frac{1}{(l+d)^{3l \left(\frac{e(l+d)}{d}\right)^{\! d}}} \nonumber \\
& = \Lambda(d,l)
\label{Eq: Intermediate eigenvalue bound}
\end{align}
where the third inequality holds because $d \geq 19$, the map $1 \leq t \mapsto t^t$ is strictly increasing, and we utilize the ensuing standard bounds on binomial coefficients:
$$ \left(\frac{l+d}{d}\right)^{\! d} \leq D = \binom{l + d}{d} \leq \left(\frac{e(l+d)}{d}\right)^{\! d} , $$
and the final equality follows from \eqref{Eq: Lower bound def}. Since $\B \succeq \lambda_{\mathsf{min}}(\B) I_D$, applying the lower bound in \eqref{Eq: Intermediate eigenvalue bound} completes the proof.
\qed
\end{proof}

We make three pertinent remarks at this point. Firstly, we note that when $d \leq l \leq \gamma d$ for some constant $\gamma > 1$, one can verify that 
\begin{equation}
\Lambda(d,l) \geq \left(\frac{\pi d}{8 e^2}\right)^{\! d 2^{d}} \frac{\left(2^{d} - 1 \right)^{2^{d} - 1}}{((\gamma+1)d)^{3 \gamma d (e(\gamma+1))^{d}}} \, . 
\end{equation}
Hence, $\Lambda(d,l)^{-1}$ scales double exponentially in $d$:
\begin{equation}
\label{Eq: DE scaling}
\frac{1}{\Lambda(d,l)} = O\!\left(\exp\!\left(e^{\Theta(\poly(d))}\right)\right) .
\end{equation}
This observation will be used in the proof of \propref{Prop: Comparison to GD and SGD} in \secref{Proof of Prop Comparison to GD and SGD}. 

Secondly, we note that the (crude) scaling with $d$ of the lower bound on $\lambda_{\mathsf{min}}(\B)$ deduced in \eqref{Eq: DE scaling} can also be derived by an alternative known approach. Define the constant $C \triangleq ((2 l + 1)!)^d (l!)^2$, and notice that $C \B$ is an non-negative integer-valued matrix due to \eqref{Eq: Entries of B}. Indeed, for any $r,s \in \Z_+^d$ with $|r|,|s| \leq l$, we have
\begin{equation}
C \B_{r,s} = \I\!\left\{r + s \text{ is even entry-wise}\right\} 2^d l(l-1) \cdots (|r|+1) l(l-1) \cdots (|s|+1) \frac{|r|!}{r!} \frac{|s|!}{r!} \prod_{j = 1}^{d}{\frac{(2 l + 1)!}{r_j + s_j + 1}} , 
\end{equation}
which belongs to $\Z_+$, because $|r|!/r!$ and $|s|!/s!$ are multinomial coefficients, and the denominator divides the numerator in each term of the product. Since $C \B$ is both positive definite and non-negative integer-valued, its determinant is a positive integer using the Leibniz formula \cite[Section 0.3.2]{HornJohnson2013}, viz. $\det(C \B) \geq 1$. Hence, the multilinearity of the determinant (cf. \cite[Section 0.3.6]{HornJohnson2013}) yields the lower bound
\begin{equation}
\label{Eq: Det lower bound 2}
\det(\B) \geq \frac{1}{C^{D}} \, . 
\end{equation}
As in the proof of \lemref{Lemma: Auxiliary Loewner Lower Bound}, combining \eqref{Eq: Trace upper bound}, \eqref{Eq: Eigenvalue lower bound}, and \eqref{Eq: Det lower bound 2}, we obtain
\begin{align}
\lambda_{\mathsf{min}}(\B) & \geq \frac{(D-1)^{D-1}}{(((2 l + 1)!)^d (l!)^2)^D (2e)^{d(D-1)}} \nonumber \\
& \geq \frac{(D-1)^{D-1}}{(2 e (2 l + 1)!)^{(d + 2)D}} \, .
\label{Eq: Intermediate eigenvalue bound 2}
\end{align}
Letting $l = \Theta(d)$ and applying Stirling's approximation to the above factorial term, one can verify that the lower bound in \eqref{Eq: Intermediate eigenvalue bound 2} also scales double exponentially in $d$. Despite the simplicity of this argument, we prefer the analysis of $\det(\B)$ in the proof of \lemref{Lemma: Auxiliary Loewner Lower Bound}, because it exactly evaluates $\det(\B)$ before lower bounding it, thereby illustrating why the derived scaling of $\det(\B)$ with $d$ and $l$ is not loose. Moreover, the lower bound on $\lambda_{\mathsf{min}}(\B)$ in \eqref{Eq: Intermediate eigenvalue bound} is generally tighter than that in \eqref{Eq: Intermediate eigenvalue bound 2}. 
 
Thirdly, we also conjecture that the scaling in \eqref{Eq: DE scaling} can be significantly improved to something like $O(d^{\Theta(\poly(d))})$. Indeed, using the \emph{Courant-Fischer-Weyl min-max theorem} \cite[Theorem 4.2.6]{HornJohnson2013} (also see the \emph{Schur-Horn theorem} \cite[Theorem 4.3.45]{HornJohnson2013}), the minimum eigenvalue of $\B$ is upper bounded by the minimum diagonal entry of $\B$ as follows:
\begin{align}
\lambda_{\mathsf{min}}(\B) & \leq \min_{s \in \Z_+^d : \, |s| \leq l}{\frac{2^d}{(s!)^2} \prod_{j = 1}^{d}{\frac{1}{2 s_j + 1}}} \nonumber \\
& \leq \frac{2^{d}}{(2 l + 1) (l!)^2} \nonumber \\
& \leq \frac{2^{d-1} e^{2 l}}{\pi (2 l + 1) l^{2l+1}} \nonumber \\
& \leq \frac{2^{d-2} e^{2 l}}{\pi l^{2l+2}} \, ,
\end{align}
where the second inequality follows from setting $s_1 = l$ and $s_2 = \cdots = s_d = 0$ (which is reasonable because it asymptotically achieves the minimum), and the third inequality follows from the Stirling's formula bound (cf. \cite[Chapter II, Section 9, Equation (9.15)]{Feller1968}). Thus, $\lambda_{\mathsf{min}}(\B)^{-1}$ has the potential to scale as
\begin{equation}
\frac{\pi l^{2l+2}}{2^{d-2} e^{2 l}} = O\!\left(d^{\Theta(\poly(d))}\right)
\end{equation}
when $l = \Theta(\poly(d))$. We believe that this scaling is likely to be close to the actual scaling of $\lambda_{\mathsf{min}}(\B)^{-1}$. Since we exactly evaluate the trace and determinant in the proof of \lemref{Lemma: Auxiliary Loewner Lower Bound}, the looseness of our minimum eigenvalue lower bound in \lemref{Lemma: Auxiliary Loewner Lower Bound} stems from the AM--GM based result in \eqref{Eq: Eigenvalue lower bound}. Hence, a fruitful future direction would be to obtain tighter bounds on the minimum eigenvalue in terms of trace and determinant that are still sufficiently tractable so as to permit analytical evaluation.  

We now establish \propref{Prop: Lowner Lower Bound} using \lemref{Lemma: Auxiliary Loewner Lower Bound}.

\begin{proof}[of \propref{Prop: Lowner Lower Bound}]
We take inspiration from the proof technique of \cite[Lemma 1.5]{Tsybakov2009}, but the details of our analysis are more involved and quite different. Observe that for all $m \in \N$, for all $x \in [h,1-h]^d$, and for all $v \in \R^{D}$ with $\|v\|_{2} = 1$, we have
\begin{align*}
v^{\T} B(x) v & = v^{\T}\left(\frac{1}{(m h)^d} \sum_{y \in \G_m}{ U\!\left(\frac{y-x}{h}\right) U\!\left(\frac{y-x}{h}\right)^{\! \T} \prod_{j = 1}^{d}{K\!\left(\frac{y_j - x_j}{h}\right)}}\right) v \\
& = v^{\T}\left(\frac{1}{(m h)^d} \sum_{\substack{z \in [-1,1]^d :\\ x + hz \in \G_m}}{ U(z) U(z)^{\T} \prod_{j = 1}^{d}{K(z_j)}}\right) v \\
& \geq b^d v^{\T} \tilde{B}(x) v 
\end{align*}
where the first equality follows from \eqref{Eq: B-matrix}, the second equality follows from the substitution $y = x + hz$ and the fact that the kernel $K$ has support $[-1,1]$, the third inequality holds because the kernel $K$ is lower bounded by $b > 0$ on $[-1,1]$, and we define $\tilde{B}(x)$ as the $D \times D$ symmetric positive semidefinite matrix
$$ \tilde{B}(x) \triangleq \sum_{\substack{z \in [-1,1]^d :\\ x + hz \in \G_m}}{U(z) U(z)^{\T} \frac{1}{(m h)^d}} \, . $$
Then, applying the \emph{Courant-Fischer-Weyl min-max theorem} \cite[Theorem 4.2.6]{HornJohnson2013}, we obtain
\begin{equation}
\label{Eq: Min Eval Domination}
\lambda_{\mathsf{min}}(B(x)) \geq b^d \lambda_{\mathsf{min}}\big(\tilde{B}(x)\big) \, .
\end{equation}

We proceed to lower bounding $\lambda_{\mathsf{min}}\big(\tilde{B}(x)\big)$. Fix any $x \in [h,1-h]^d$, and any $r,s \in \Z_+^d$ with $|r|,|s| \leq l$, and consider the Riemann sum
\begin{align*}
\tilde{B}(x)_{r,s} & = \frac{1}{r! s!}\sum_{\substack{z \in [-1,1]^d :\\ x + hz \in \G_m}}{\frac{z^{r+s}}{(m h)^d}} \, . 
\end{align*}
Since $x \in [h,1-h]^d$, it is easy to verify that the points $\{z \in [-1,1]^d : x + hz \in \G_m\}$ partition the hypercube $[-1,1]^d$ into a grid of $d$-cells $\{I_z \subseteq [-1,1]^d : z \in [-1,1]^d \, , x + hz \in \G_m\}$ with side lengths $(m h)^{-1}$ and volume $(m h)^{-d}$ each, where $I_z$ is the $d$-cell with tag $z$. Moreover, since the polynomial function $\gamma:[-1,1]^d \rightarrow \R$, $\gamma(u) = u^{r + s}$ is Riemann integrable, if $h = \omega(m^{-1})$ with all other problem parameters fixed, then the above Riemann sum converges to an entry of the matrix $\B$ defined in \lemref{Lemma: Auxiliary Loewner Lower Bound} (cf. \cite[Section 14.3]{Apostol1974}):
$$ \lim_{m \rightarrow \infty}{\frac{1}{r! s!} \sum_{\substack{z \in [-1,1]^d :\\ x + hz \in \G_m}}{\frac{z^{r+s}}{(m h)^d}}} = \int_{[-1,1]^d}{\frac{u^{r+s}}{r! s!}  \diff{u}} = \B_{r,s} \, . $$
Therefore, we have shown that $\lim_{m \rightarrow \infty}{\tilde{B}(x)_{r,s}} = \B_{r,s}$ when $h = \omega(m^{-1})$. We will require an explicit upper bound on the rate of this convergence in the sequel. To this end, let us first compute the Lipschitz constant of $\gamma$ with respect to the $\ell^{\infty}$-norm. For any $i \in [d]$ and any fixed $(u_1,\dots,u_{i-1},u_{i+1},\dots,u_d) \in [-1,1]^{d-1}$, notice that for every $u_i,\tilde{u}_i \in [-1,1]$,
\begin{align*}
& \left|\gamma(u) - \gamma(u_1,\dots,u_{i-1},\tilde{u}_i,u_{i+1},\dots,u_d)\right| \\
& \qquad \qquad \qquad = \left|u_1^{r_1 + s_1} \cdots u_{i-1}^{r_{i-1}+s_{i-1}} u_{i+1}^{r_{i+1} + s_{i+1}} \cdots u_d^{r_d + s_d}\right| \left|u_i^{r_i + s_i} - \tilde{u}_i^{r_i + s_i}\right| \\
& \qquad \qquad \qquad \leq (r_i + s_i) |u_i - \tilde{u}_i| \, .
\end{align*}
Hence, via repeated application of the triangle inequality, we have
\begin{align}
\forall u,\tilde{u} \in [-1,1]^{d}, \enspace \left|\gamma(u) - \gamma(\tilde{u})\right| = \left|u^{r+s} - \tilde{u}^{r+s}\right| & \leq \sum_{i = 1}^{d}{(r_i + s_i) |u_i - \tilde{u}_i|} \nonumber \\
& \leq (|r|+|s|) \left\|u - \tilde{u}\right\|_{\infty} .
\label{Eq: L^inf Lip Cond}
\end{align}
Now observe that for all $r,s \in \Z_+^d$ with $|r|,|s| \leq l$,
\begin{align*}
\left|\tilde{B}(x)_{r,s} - \B_{r,s} \right| & = \frac{1}{r! s!} \left|\sum_{\substack{z \in [-1,1]^d :\\ x + hz \in \G_m}}{\frac{z^{r+s}}{(m h)^d} - \int_{I_z}{u^{r+s}  \diff{u}}} \right| \\ 
& \leq \frac{1}{r! s!} \sum_{\substack{z \in [-1,1]^d :\\ x + hz \in \G_m}}{\left|\frac{z^{r+s}}{(m h)^d} - \int_{I_z}{u^{r+s} \diff{u}}\right|} \\
& \leq \frac{1}{r! s!} \sum_{\substack{z \in [-1,1]^d :\\ x + hz \in \G_m}}{\frac{1}{(m h)^d}\left(\sup_{u \in I_z}{u^{r+s}} - \inf_{\tilde{u} \in I_z}{\tilde{u}^{r+s}}\right)} \\
& \leq \frac{|r|+|s|}{m h r! s!} \sum_{\substack{z \in [-1,1]^d :\\ x + hz \in \G_m}}{\frac{1}{(m h)^{d}}} \\
& \leq \frac{|r|+|s|}{mh r! s!} \left(2 + \frac{1}{m h}\right)^{\! d} \\
& \leq \frac{3^d (|r|+|s|)}{mh r! s!} \, , 
\end{align*}
where the second inequality follows from the triangle inequality, the fourth inequality follows from \eqref{Eq: L^inf Lip Cond}, the fifth inequality holds because $|\{z \in [-1,1]^d : x + hz \in \G_m\}| \leq (2 m h + 1)^d$ (since the interval $[-1,1]$ can have at most $2mh + 1$ points at distance $(mh)^{-1}$ from each other), and the sixth inequality holds because we assume that $m \in \N$ and $h \in \big(0,\frac{1}{2}\big)$ are such that $m h \geq 4 l (3e)^d/\Lambda(d,l) \geq 1$. This implies that for all such $m$ and $h$, 
\begin{align}
\left\| \tilde{B}(x) - \B \right\|_{\F} & \leq \sqrt{\sum_{r,s \in \Z_+^d: \, |r|,|s|\leq l}{\frac{9^d (|r|+|s|)^2}{m^2 h^2 (r! s!)^2}}} \nonumber \\
& \leq \frac{2 \, l 3^d}{m h} \sum_{r \in \Z_+^d: \, |r|\leq l}{\frac{1}{(r!)^2}} \nonumber \\
& \leq \frac{2 \, l 3^d}{m h} \sum_{j = 0}^{l}{\frac{1}{j!}\sum_{r \in \Z_+^d: \, |r| = j}{\frac{j!}{r!}}} \nonumber \\
& = \frac{2 \, l 3^d}{m h} \sum_{j = 0}^{l}{\frac{d^j}{j!}} 
\nonumber \\
& \leq \frac{2 \, l (3e)^d}{m h}
\label{Eq: Fro bound}
\end{align}
where the fourth equality follows from the multinomial theorem, and the fifth inequality uses the Maclaurin series of the exponential function. Next, recall the following consequence of the \emph{Wielandt-Hoffman-Mirsky inequality}, cf. \cite[Corollary 6.3.8]{HornJohnson2013}:
\begin{equation}
\label{Eq: WHM Inequality}
\left|\lambda_{\mathsf{min}}\big(\tilde{B}(x)\big) - \lambda_{\mathsf{min}}(\B)\right| \leq \left\| \tilde{B}(x) - \B \right\|_{\F} . 
\end{equation}
Then, combining \eqref{Eq: Fro bound} and \eqref{Eq: WHM Inequality} produces the lower bound
$$ \lambda_{\mathsf{min}}\big(\tilde{B}(x)\big) \geq \lambda_{\mathsf{min}}(\B) - \frac{2 \, l (3e)^d}{m h} \, , $$
and employing \lemref{Lemma: Auxiliary Loewner Lower Bound} to this bound yields
\begin{equation}
\label{Eq: Intermediate LB on Eval}
\lambda_{\mathsf{min}}\big(\tilde{B}(x)\big) \geq \Lambda(d,l) - \frac{2 \, l (3e)^d}{m h} \, , 
\end{equation}
where $\Lambda(d,l)$ is defined in \eqref{Eq: Lower bound def}.

Finally, since $m$ and $h$ satisfy $m h \geq 4 l (3e)^d/\Lambda(d,l)$, or equivalently,  
$$ \frac{2 \, l (3e)^d}{m h} \leq \frac{\Lambda(d,l)}{2} \, , $$
using \eqref{Eq: Min Eval Domination} and \eqref{Eq: Intermediate LB on Eval}, we get
$$ \lambda_{\mathsf{min}}(B(x)) \geq b^d \lambda_{\mathsf{min}}\big(\tilde{B}(x)\big) \geq \frac{b^d \Lambda(d,l)}{2} $$
for all $x \in [h,1-h]^d$. Since $B(x) \succeq \lambda_{\mathsf{min}}(B(x)) I_D$, the above lower bound on $\lambda_{\mathsf{min}}(B(x))$ completes the proof.
\qed
\end{proof}

\subsection{Properties of Local Polynomial Interpolation Weights}
\label{Properties of Local Polynomial Interpolation Weights}

To prove \thmref{Thm: Chebyshev Norm Interpolation}, we will require two useful lemmata regarding the interpolation weights in \eqref{Eq: Weights}. The first lemma generalizes \cite[Proposition 1.12]{Tsybakov2009} to the $d$-variate setting, and demonstrates the intuitively pleasing property that polynomials are reproduced exactly by the local polynomial interpolator given in \eqref{Eq: Local polynomial interpolator}.

\begin{lemma}[Polynomial Reproduction {\cite[Proposition 1.12]{Tsybakov2009}}]
\label{Lemma: Polynomial Reproduction}
Consider any $d$-variate polynomial $g : [0,1]^d \rightarrow \R$ with total degree at most $l$ (i.e., for every monomial of $g$ with non-zero coefficient, the sum of the exponents of the $d$ variables is at most $l$). Then, for every $m \in \N$ and every $x \in [h,1-h]^d$ such that $B(x)$ is positive definite, we have
$$ \hat{\phi}(x) = \sum_{y \in \G_m}{g(y) w_y^*(x)} = g(x) \, , $$
where the interpolation weights are given in \eqref{Eq: Weights}.
\end{lemma}

\begin{proof}
We follow the proof of \cite[Proposition 1.12]{Tsybakov2009} mutatis mutandis. Fix any $m \in \N$ and $x \in [h,1-h]^d$ such that $B(x)$ is positive definite, and define the vector $q(x)\in \R^D$ as
$$ q(x) \triangleq \left[h^{|s|} \nabla^s g (x) : s \in \Z_+^d, \, |s| \leq l\right]^{\! \T} . $$
Then, for every $y \in \G_m$, we have
\begin{align*}
g(y) & = \sum_{s \in \Z_+^d: \, |s| \leq l}{\frac{\nabla^s g (x)}{s!} (y-x)^s} \\
& = q(x)^{\T} U\!\left(\frac{y-x}{h}\right) ,
\end{align*}
where the first equality uses Taylor's theorem (cf. \cite[Theorem 12.14]{Apostol1974}) and the fact that $g$ is a polynomial with total degree at most $l$, and $U$ is defined in \eqref{Eq: Def of U}. This implies that the solution to the weighted regression problem in \eqref{Eq: Regression} satisfies
\begin{align*}
\hat{\Phi}(x) & = \argmin{\phi \in \R^D} \sum_{y \in \G_m}{(q(x) - \phi)^{\T} U\!\left(\frac{y-x}{h}\right) U\!\left(\frac{y-x}{h}\right)^{\! \T} (q(x) - \phi) \prod_{j=1}^{d}{K\!\left(\frac{y_j - x_j}{h}\right)}} \\
& = \argmin{\phi \in \R^D} (q(x) - \phi)^{\T} B(x) (q(x) - \phi) \\
& = q(x)
\end{align*}
where the second equality follows from \eqref{Eq: B-matrix}, and the third equality holds because $B(x)$ is positive definite. Therefore, by considering the first coordinate, we get
$$ \hat{\phi}(x) = \sum_{y \in \G_m}{g(y) w_y^*(x)} = g(x) \, , $$
where the first equality uses \eqref{Eq: Local polynomial interpolator}. This completes the proof.
\qed
\end{proof}

The second lemma generalizes \cite[Lemma 1.3]{Tsybakov2009} to the $d$-variate setting, and portrays an $\ell^1$-norm bound on the sequence of interpolation weights in \eqref{Eq: Weights}.

\begin{lemma}[$\ell^1$-Norm Bound]
\label{Lemma: 1-Norm Bound}
Suppose that $l \geq d \geq 19$. Then, for all $m \in \N$ and $h \in \big(0,\frac{1}{2}\big)$ such that $m h \geq 4 l (3e)^d/\Lambda(d,l)$ and all $x \in [h,1-h]^d$, we have
$$ \sum_{y \in \G_m}{\left|w_y^*(x)\right|} \leq \frac{2}{\Lambda(d,l)} \left(\frac{3 \sqrt{e} \, c}{b}\right)^{\! d} , $$
where $c\geq b > 0$ are the bounds on our kernel, and $\Lambda(d,l)$ is defined in \eqref{Eq: Lower bound def}.
\end{lemma}

\begin{proof}
Once again, we follow the proof of \cite[Lemma 1.3]{Tsybakov2009} mutatis mutandis. First, notice that $B(x)$ satisfies the L\"{o}wner lower bound in \propref{Prop: Lowner Lower Bound} since the conditions of the proposition are satisfied. Then, observe that for every $x \in [h,1-h]^d$, we have
\begin{align*}
\sum_{y \in \G_m}{\left|w_y^*(x)\right|} & \leq \sum_{y \in \G_m}{\left\|w_y(x)\right\|_2} \\
& \leq \frac{1}{(mh)^d} \sum_{y \in \G_m}{\left\|B(x)^{-1} U\!\left(\frac{y-x}{h}\right)\right\|_2 \prod_{j = 1}^{d}{K\!\left(\frac{y_j - x_j}{h}\right)}} \\
& \leq \frac{c^d}{(mh)^d} \sum_{y \in \G_m}{\left\|B(x)^{-1} U\!\left(\frac{y-x}{h}\right)\right\|_2 \I\!\left\{\forall j \in [d], \, |y_j - x_j| \leq h\right\}} \\
& \leq \frac{2 c^d}{(mhb)^d \Lambda(d,l)} \sum_{y \in \G_m}{\left\|U\!\left(\frac{y-x}{h}\right)\right\|_2 \I\!\left\{\forall j \in [d], \, |y_j - x_j| \leq h\right\}} \\
& \leq \frac{2}{\Lambda(d,l)} \left(\frac{\sqrt{e} \, c}{mhb}\right)^{\! d} \sum_{y \in \G_m}{\I\!\left\{\forall j \in [d], \, |y_j - x_j| \leq h\right\}} \\
& \leq \frac{2}{\Lambda(d,l)} \left(\frac{\sqrt{e} \, c}{b}\right)^{\! d} \left(\frac{2mh + 1}{mh}\right)^{\! d} \\
& \leq \frac{2}{\Lambda(d,l)} \left(\frac{3\sqrt{e} \, c}{b}\right)^{\! d} 
\end{align*}
where the first inequality follows from \eqref{Eq: Weights} and the Cauchy-Schwarz inequality, the second equality follows from \eqref{Eq: Vector Weights}, the third inequality holds because the kernel $K$ is upper bounded by $c > 0$ on the support $[-1,1]$, the fourth inequality follows from \propref{Prop: Lowner Lower Bound}, the fifth inequality follows from the bound
\begin{align*}
\left\|U\!\left(\frac{y-x}{h}\right)\right\|_2^2 & = \sum_{s \in \Z_+^d : \, |s| \leq l}{\left(\frac{(y-x)^s}{h^{|s|} s!}\right)^{\! 2}} \\
& \leq \sum_{s \in \Z_+^d : \, |s| \leq l}{\frac{1}{(s!)^2}} \\
& \leq \sum_{i = 0}^{l}{\frac{1}{i!}\sum_{s \in \Z_+^d : \, |s| = i}{\frac{i!}{s!}}} \\
& = \sum_{i = 0}^{l}{\frac{d^i}{i!}} \\
& \leq e^d \, ,
\end{align*}
which uses \eqref{Eq: Def of U}, the fact that $|y_j - x_j| \leq h$ for all $j \in [d]$, the multinomial theorem, and the Maclaurin series of the exponential function, the sixth inequality holds because the interval $[-h,h]$ can have at most $2mh + 1$ points at distance $m^{-1}$ from each other, and the seventh inequality follows from the assumption that $mh \geq 4 l (3e)^d/\Lambda(d,l) \geq 1$. This completes the proof.
\qed
\end{proof}

\subsection{Proof of \thmref{Thm: Chebyshev Norm Interpolation}}
\label{Proof of Theorem Chebyshev Norm Interpolation}

Finally, we are in a position to prove the supremum norm interpolation guarantee in \thmref{Thm: Chebyshev Norm Interpolation}.

\begin{proof}[of \thmref{Thm: Chebyshev Norm Interpolation}]
Suppose that $m \in \N$ and $h \in \big(0,\frac{1}{2}\big)$ are such that $m h \geq 4 l (3e)^d / \Lambda(d,l)$. Then, observe that for every $x \in [h,1-h]^d$, we have
\begin{align*}
\hat{\phi}(x) - g(x) & = \sum_{y \in \G_m}{g(y) w_y^*(x)} - g(x) \\
& = \sum_{y \in \G_m}{w_y^*(x) (g(y) - g(x))} \\
& = \sum_{y \in \G_m} w_y^*(x) \sum_{s \in \Z_+^d : \, 1 \leq |s| \leq l}{\frac{\nabla^s g (x)}{s!} (y - x)^s} \\
& \qquad \quad \enspace + w_y^*(x) \sum_{s \in \Z_+^d : \, |s| = l}{\frac{\nabla^s g (x + \tau (y - x)) - \nabla^s g (x)}{s!} (y - x)^s}  \\
& = \sum_{y \in \G_m}{w_y^*(x) \sum_{s \in \Z_+^d : \, |s| = l}{\frac{\nabla^s g (x + \tau (y - x)) - \nabla^s g (x)}{s!} (y - x)^s}}
\end{align*}
where the first equality follows from \eqref{Eq: Local polynomial interpolator}, the second equality holds because
$$ \sum_{y \in \G_m}{w_y^*(x)} = 1 $$
which follows from \lemref{Lemma: Polynomial Reproduction} applied to the constant unit polynomial, the third equality follows from Taylor's theorem with Lagrange remainder term \cite[Theorem 12.14]{Apostol1974}:
$$ g(y) - g(x) = \sum_{s \in \Z_+^d : \, 1 \leq |s| \leq l}{\frac{\nabla^s g (x)}{s!} (y - x)^s} + \sum_{s \in \Z_+^d : \, |s| = l}{\frac{\nabla^s g (x + \tau (y - x)) - \nabla^s g (x)}{s!} (y - x)^s} $$
which holds for some $\tau \in (0,1)$, and the fourth equality holds because 
$$ \sum_{y \in \G_m}{(y - x)^s w_y^*(x)} = 0 $$
for all $s \in \Z_+^d$ with $1 \leq |s| \leq l$, which follows from \lemref{Lemma: Polynomial Reproduction} applied to the polynomials $u \mapsto (u - x)^s$. This implies that
\begin{align*}
\sup_{x \in [h,1-h]^d}{\left|\hat{\phi}(x) - g(x)\right|} & \leq \sum_{y \in \G_m}{\left|w_y^*(x)\right| \left|\sum_{s \in \Z_+^d : \, |s| = l}{\frac{\nabla^s g (x + \tau (y - x)) - \nabla^s g (x)}{s!} (y - x)^s}\right|} \\
& \leq \frac{L_2}{l!} \sum_{y \in \G_m}{\left|w_y^*(x)\right| \left\|y - x\right\|_1^{\eta}} \\
& = \frac{L_2}{l!} \sum_{y \in \G_m}{\left|w_y^*(x)\right| \left\|y - x\right\|_1^{\eta} \I\!\left\{\forall j \in [d], \, |y_j - x_j| \leq h\right\}} \\
& \leq \frac{2 L_2 (dh)^{\eta}}{l! \Lambda(d,l)} \left(\frac{3 \sqrt{e} \, c}{b}\right)^{\! d} \\
& = \frac{2 L_2}{l! \Lambda(d,l)} \left(\frac{3 \sqrt{e} \, c}{b}\right)^{\! d}  \left(\frac{4 d l (3e)^d}{m \Lambda(d,l)}\right)^{\! \eta}
\end{align*}
where the first inequality follows from the triangle inequality, the second inequality follows from the proof of \propref{Prop: Taylor Approximation of Holder Class Functions} in \appref{Taylor Approximation of Holder Class Functions}, the third equality follows from \eqref{Eq: Vector Weights} and \eqref{Eq: Weights} since the kernel $K$ has support $[-1,1]$, the fourth inequality follows from \lemref{Lemma: 1-Norm Bound}, and the fifth equality holds because we set
\begin{equation}
\label{Eq: Bandwidth choice}
h = \frac{4 l (3e)^d}{m \Lambda(d,l)} \, , 
\end{equation}
which minimizes the upper bound in the fourth inequality and achieves equality in the condition imposed on $m$ and $h$ at the outset of the proof. Lastly, to ensure that $\sup_{x \in [h,1-h]^d}{|\hat{\phi}(x) - g(x)|} \leq \delta$, it suffices to impose the condition
$$ \frac{2 L_2}{l! \Lambda(d,l)} \left(\frac{3 \sqrt{e} \, c}{b}\right)^{\! d}  \left(\frac{4 d l (3e)^d}{m \Lambda(d,l)}\right)^{\! \eta} \leq \delta \, , $$
which is equivalent to 
$$ m \geq \frac{4 d l (3e)^d}{\Lambda(d,l)} \left(\frac{2 L_2}{l! \Lambda(d,l)}\right)^{\! 1/\eta} \left(\frac{3 \sqrt{e} \, c}{b}\right)^{\! d/\eta} \left(\frac{1}{\delta}\right)^{\! 1/\eta} , $$
which in turn is implied by the condition
\begin{equation}
\label{Eq: Unsimplified condition}
m \geq \frac{12 \sqrt{e} (2 L_2 + 1) c}{b} \left(\frac{d l (3e)^d}{(l!)^{1/\eta} \Lambda(d,l)^{2}}\right) \left(\frac{1}{\delta}\right)^{\! 1/\eta} , 
\end{equation}
where we use the facts that $\eta > l \geq d \geq 19$ (by assumption) and $\Lambda(d,l)^{(\eta + 1)/\eta} \geq \Lambda(d,l)^2$ (since $\Lambda(d,l) \geq 1$), and we slacken the $\Lambda(d,l)$ term because this does not change its double exponential nature. Moreover, we may further simplify \eqref{Eq: Unsimplified condition} by noting that
$$ (l!)^{1/\eta} \geq \frac{(2 \pi)^{1/(2\eta)} l^{(l/\eta) + (1/(2\eta))}}{e^{l/\eta}} \geq \frac{l}{e l^{\frac{1}{l+1}}} \geq \frac{l}{2e} $$
where we use the Stirling's formula bound (cf. \cite[Chapter II, Section 9, Equation (9.15)]{Feller1968}) and the fact that $l^{1/(l+1)} \leq 2$ for $l \geq 19$. Applying this bound, we obtain the sufficient condition for \eqref{Eq: Unsimplified condition} presented in the theorem statement:
$$ m \geq \frac{110 (2 L_2 + 1) c}{b} \left(\frac{d (3e)^d}{\Lambda(d,l)^{2}}\right) \! \left(\frac{1}{\delta}\right)^{\! 1/\eta} , $$
where we use the fact that $24 e \sqrt{e} \leq 110$. Finally, plugging this bound into \eqref{Eq: Bandwidth choice} yields the upper bound on $h$ in the theorem statement. This completes the proof.
\qed
\end{proof}

We note that this argument is inspired by the proof of \cite[Proposition 1.13]{Tsybakov2009}, but the details are different.

\section{Proofs of Convergence Analysis}
\label{Proofs of Convergence Analysis}

In this section, we prove \propref{Prop: Number of Iterations}, \thmref{Thm: Oracle Complexity}, and \propref{Prop: Comparison to GD and SGD} from \secref{Analysis of Algorithm}.

\subsection{Proof of \propref{Prop: Number of Iterations}}
\label{Proof of Prop Number of Iterations}

To establish \propref{Prop: Number of Iterations}, we require the following lemma from the literature \cite[Lemma 2.1]{FriedlanderSchmidt2012} (also see \cite[Theorem 2, Equation (19)]{SoZhou2017}) that upper bounds the difference between the ERM objective function \eqref{Eq: ERM Objective Function} evaluated at $\theta^{(T)}$, which is the output of our proposed algorithm, and the true infimum in \eqref{Eq: ERM}. Although \cite[Lemma 2.1]{FriedlanderSchmidt2012} holds for general inexact gradient descent methods, we unwind the recursion in \cite[Lemma 2.1]{FriedlanderSchmidt2012} and adapt it to our setting below for convenience.

\begin{lemma}[Inexact Gradient Descent Bound {\cite[Lemma 2.1]{FriedlanderSchmidt2012}}]
\label{Lemma: Inexact GD Bound}
Suppose the ERM objective function $F : \R^p \rightarrow \R$ is continuously differentiable and $\mu$-strongly convex, and its gradient $\nabla_{\theta} F : \R^p \rightarrow \R^p$ is $L_1$-Lipschitz continuous. Then, for any $T \in \N$, we have
$$ 0 \leq F(\theta^{(T)}) - F_* \leq \left(1 - \frac{1}{\sigma}\right)^{\! T} \!\left(F(\theta^{(0)}) - F_*\right) + \frac{1}{2 L_1} \sum_{t = 1}^{T}{\left(1 - \frac{1}{\sigma}\right)^{\! T-t} \left\|\widehat{\nabla F}^{(t)} - \nabla_{\theta} F(\theta^{(t-1)})\right\|_2^2} \, , $$
where $\theta^{(t)}$ for $t \in \Z_+$ are the inexact gradient descent updates given by \eqref{Eq: Update step} with an arbitrary initialization $\theta^{(0)}$, $\widehat{\nabla F}^{(t)}$ for $t \in \N$ are our local polynomial interpolation based approximations of the true gradients $\nabla_{\theta} F(\theta^{(t-1)})$ shown in \eqref{Eq: Approximation to Grad F}, and $\sigma = L_1/\mu \geq 1$ is the condition number defined in \secref{Smoothness Assumptions and Approximate Solutions}.
\end{lemma}

We remark that \cite[Lemma 2.1]{FriedlanderSchmidt2012} additionally assumes that the infimum $F_*$ in \eqref{Eq: ERM} can be achieved. However, this condition need not be imposed in \lemref{Lemma: Inexact GD Bound}, since if $F$ is strongly convex on $\R^m$, its infimum $F_*$ is finite and achieved by some unique global minimizer $\theta_* \in \R^p$. We next prove \propref{Prop: Number of Iterations} using \lemref{Lemma: Inexact GD Bound} and \thmref{Thm: Chebyshev Norm Interpolation}.

\begin{proof}[of \propref{Prop: Number of Iterations}]
We first verify that the conditions of \lemref{Lemma: Inexact GD Bound} hold. Recall the ERM objective function $F : \R^p \rightarrow \R$ given in \eqref{Eq: ERM Objective Function}:
$$ \forall \theta \in \R^p, \enspace F(\theta) = \frac{1}{n} \sum_{i = 1}^{n}{f(x^{(i)};\theta)} \, , $$
where $\{x^{(i)} \in [h^{\prime},1-h^{\prime}]^d : i \in [n]\}$ is some fixed training data, and $f : [0,1]^d \times \R^p \rightarrow \R$ is a given loss function satisfying the assumptions in \secref{Smoothness Assumptions and Approximate Solutions}. Specifically, since $f(x^{(i)};\cdot) : \R^p \rightarrow \R$ is $\mu$-strongly convex for all $i \in [n]$, $F$ is also $\mu$-strongly convex due to \cite[Lemma 2.1.4]{Nesterov2004}. Likewise, since $\nabla_{\theta} f(x^{(i)};\cdot) : \R^p \rightarrow \R^p$ is $L_1$-Lipschitz continuous for all $i \in [n]$, $F$ is continuously differentiable and $\nabla_{\theta} F: R^p \rightarrow \R^p$ is also $L_1$-Lipschitz continuous. Indeed, observe that for all $\theta_1,\theta_2 \in \R^p$, 
\begin{align*}
\left\|\nabla_{\theta} F (\theta_1) - \nabla_{\theta} F (\theta_2) \right\|_2 & \leq \frac{1}{n} \sum_{i = 1}^{n}{ \left\|\nabla_{\theta} f (x^{(i)};\theta_1) - \nabla_{\theta} f (x^{(i)};\theta_2)\right\|_2 } \\
& \leq \frac{L_1}{n} \sum_{i = 1}^{n}{ \left\|\theta_1 - \theta_2\right\|_2 } \\
& = L_1 \left\|\theta_1 - \theta_2\right\|_2 ,
\end{align*}
where we use the triangle inequality. 

Next, we apply \lemref{Lemma: Inexact GD Bound} and get
$$ F(\theta^{(T)}) - F_* \leq \left(1 - \frac{1}{\sigma}\right)^{\! T} \! \left(F(\theta^{(0)}) - F_*\right) + \frac{1}{2 L_1} \sum_{t = 1}^{T}{\left(1 - \frac{1}{\sigma}\right)^{\! T-t} \left\|\widehat{\nabla F}^{(t)} - \nabla_{\theta} F(\theta^{(t-1)})\right\|_2^2} \, , $$
where $\theta^{(t)}$ for $t \in \Z_+$ are our updates in \eqref{Eq: Update step}, and $\widehat{\nabla F}^{(t)}$ for $t \in \N$ are our approximations of the true gradients in \eqref{Eq: Approximation to Grad F}. To upper bound the squared $\ell^2$-norms on the right hand side of this inequality, notice that for any $t \in \N$, we have
\begin{align}
\left\|\widehat{\nabla F}^{(t)} - \nabla_{\theta} F(\theta^{(t-1)})\right\|_2 & \leq \frac{1}{n} \sum_{j = 1}^{n}{\sqrt{\sum_{i = 1}^{p}{\left(\hat{\phi}_i^{(t)}(x^{(j)}) - \frac{\partial f}{\partial \theta_i}(x^{(j)};\theta^{(t-1)}) \right)^{\!2}}}} \nonumber \\
& \leq \sqrt{\sum_{i = 1}^{p}{\sup_{x \in [h^{\prime},1-h^{\prime}]^d}{\left|\hat{\phi}_i^{(t)}(x) - \frac{\partial f}{\partial \theta_i}(x;\theta^{(t-1)})\right|^{2}} }} \nonumber \\
& \leq \sqrt{p} \, \delta_t \, ,
\label{Eq: p dependence}
\end{align} 
where the first inequality follows from \eqref{Eq: Approximation to Grad F} and the triangle inequality, $\hat{\phi}_i^{(t)}(x)$ denotes the output of our local polynomial interpolator for $\frac{\partial f}{\partial \theta_i}(\cdot;\theta^{(t-1)})$ when evaluated at $x$ (cf. \eqref{Eq: Interpolators}), and the third inequality follows from \thmref{Thm: Chebyshev Norm Interpolation} and our choices of grid sizes \eqref{Eq: Algo grid sizes} and bandwidths \eqref{Eq: Special bandwidth selection} in \secref{Description of Algorithm} (where \eqref{Eq: Data bound} ensures that $h_t \leq h^{\prime}$). Hence, we obtain
\begin{equation}
\label{Eq: Bound on Diff to Optimality}
F(\theta^{(T)}) - F_* \leq \left(1 - \frac{1}{\sigma}\right)^{\! T} \! \left(F(\theta^{(0)}) - F_*\right) + \frac{p}{2 L_1} \sum_{t = 1}^{T}{\left(1 - \frac{1}{\sigma}\right)^{\! T-t}\delta_t^2} \, . 
\end{equation}

We now select appropriate values for the constants $\{\delta_t \in (0,1) : t \in \N\}$. For every $t \in \N$, choose
$$ \delta_t = \left(1 - \frac{1}{\sigma}\right)^{\!\frac{T}{2}} $$
as indicated in \eqref{Eq: Choices of the Chebyshev constants}. This produces
\begin{align*}
F(\theta^{(T)}) - F_* & \leq \left(1 - \frac{1}{\sigma}\right)^{\! T} \! \left(F(\theta^{(0)}) - F_* + \frac{p}{2 L_1} \sum_{t = 1}^{T}{\left(1 - \frac{1}{\sigma}\right)^{\! T-t}}\right) \\
& \leq \left(1 - \frac{1}{\sigma}\right)^{\! T} \! \left(F(\theta^{(0)}) - F_* + \frac{p}{2 L_1} \sum_{t = 0}^{\infty}{\left(1 - \frac{1}{\sigma}\right)^{\! t}}\right) \\
& = \left(1 - \frac{1}{\sigma}\right)^{\! T} \! \left(F(\theta^{(0)}) - F_* + \frac{p \sigma}{2 L_1}\right) ,
\end{align*}
where the last equality computes the value of the previous geometric series.

Finally, to ensure that $\theta^{(T)}$ is an $\epsilon$-approximate solution of \eqref{Eq: ERM}, viz.
$$ F(\theta^{(T)}) - F_* \leq \epsilon $$
as defined in \eqref{Eq: Minimum of convex function}, it suffices to impose the condition
$$ \left(1 - \frac{1}{\sigma}\right)^{\! T} \! \left(F(\theta^{(0)}) - F_* + \frac{p \sigma}{2 L_1}\right) \leq \epsilon \, . $$
By taking logarithms of both sides, and rearranging and simplifying the resulting inequality, we obtain the equivalent inequality
$$ T \geq \left(\log\!\left(\frac{\sigma}{\sigma - 1}\right)\right)^{\! -1} \log\!\left(\frac{F(\theta^{(0)}) - F_* + \frac{p}{2 \mu}}{\epsilon}\right) , $$
where we also use the fact that $\sigma = L_1/\mu$. Therefore, after running the LPI-GD algorithm for $T \in \N$ iterations with $T$ given by \eqref{Eq: Chosen number of iterations}, the updated parameter vector $\theta^{(T)}$ forms an $\epsilon$-approximate solution to \eqref{Eq: ERM}. We also note that for every $t \in \N$, 
$$ \delta_t = \left(1 - \frac{1}{\sigma}\right)^{\!\frac{T}{2}} = \Theta\!\left(\sqrt{\frac{\epsilon}{p}}\right) , $$
by substituting \eqref{Eq: Chosen number of iterations} into our choice of $\delta_t$ above. This yields \eqref{Eq: Choices of the Chebyshev constants}, and completes the proof.
\qed
\end{proof}

\subsection{Proof of \thmref{Thm: Oracle Complexity}}
\label{Proof of Thm Oracle Complexity}

We next establish \thmref{Thm: Oracle Complexity} using \propref{Prop: Number of Iterations}.

\begin{proof}[of \thmref{Thm: Oracle Complexity}]
Since we choose the maximum allowable approximation errors $\{\delta_t \in (0,1) : t \in [T]\}$ to be constant with respect to $t$ as in \eqref{Eq: Choices of the Chebyshev constants}, a single common uniform grid $\G_m$ is used in all iterations of the LPI-GD algorithm, where $m = m_t$ is given by \eqref{Eq: Algo grid sizes} (and is also constant with respect to $t$). The size of this grid is upper bounded by
\begin{align*}
|\G_m| & = m^d \\
& \leq \left( \frac{110 (2 L_2 + 1) c}{b} \left(\frac{d (3e)^d}{\Lambda(d,l)^{2}}\right) \! \left(\frac{\sigma}{\sigma - 1}\right)^{\! T/(2\eta)} + 1 \right)^{\! d} \\
& \leq \left( \frac{110 (2 L_2 + 1) c}{b} \left(\frac{d (3e)^d}{\Lambda(d,l)^{2}}\right) \! \left(\frac{\sigma}{\sigma - 1}\right)^{1/(2\eta)} \! \left(\frac{F(\theta^{(0)}) - F_* + \frac{p}{2 \mu}}{\epsilon}\right)^{\! 1/(2\eta)} + 1 \right)^{\! d} \\
& \leq \left( \frac{220 (2 L_2 + 1) c}{b} \right)^{\! d} \! \left(\frac{d^d (3e)^{d^2}}{\Lambda(d,l)^{2 d}}\right) \! \left(1 + \frac{\sigma}{2 (L_1 - \mu)}\right)^{d/(2\eta)} \! \left(\frac{p + 2 \mu (F(\theta^{(0)}) - F_*)}{\epsilon}\right)^{\! d/(2\eta)} \\
& \leq \left(1 + \frac{\sigma}{2 (L_1 - \mu)}\right) \! \left( \frac{220 (2 L_2 + 1) c}{b} \right)^{\! d} \! \left(\frac{d^d (3e)^{d^2}}{\Lambda(d,l)^{2 d}}\right) \! \left(\frac{p + 2 \mu (F(\theta^{(0)}) - F_*)}{\epsilon}\right)^{\! d/(2\eta)} ,
\end{align*}
where the second inequality follows from \eqref{Eq: Algo grid sizes} and \eqref{Eq: Choices of the Chebyshev constants} in \propref{Prop: Number of Iterations}, the third inequality follows from \eqref{Eq: Chosen number of iterations} in \propref{Prop: Number of Iterations}, the fourth inequality follows from straightforward manipulations and the fact that $\sigma = L_1/\mu$, and the fifth inequality holds because we have assumed that $\eta > d$. The LPI-GD algorithm makes $|\G_m|$ oracle calls in each of its $T$ iterations and produces an $\epsilon$-approximate solution $\theta^{(T)} \in \R^p$ (in the sense of \eqref{Eq: Minimum of convex function}) as argued in \propref{Prop: Number of Iterations}. Hence, the oracle complexity of this algorithm is upper bounded by
\begin{align*}
\Gamma(\text{LPI-GD}) & \leq T |\G_m| \\
& \leq \left(\frac{\sigma + 2(L_1 - \mu)}{(L_1 - \mu)\log\!\left(\frac{\sigma}{\sigma - 1}\right)}\right)\! \log\!\left(\frac{p + \Delta}{2 \mu \epsilon}\right) \! \left( \frac{220 (2 L_2 + 1) c}{b} \right)^{\! d} \! \left(\frac{d^d (3e)^{d^2}}{\Lambda(d,l)^{2 d}}\right) \! \left(\frac{p + \Delta}{\epsilon}\right)^{\! d/(2\eta)} \\
& = C_{\mu,L_1,L_2,b,c}(d,l) \left(\frac{p + \Delta}{\epsilon}\right)^{\! d/(2\eta)} \log\!\left(\frac{p + \Delta}{2 \mu \epsilon}\right)
\end{align*}
where we let $\Delta = 2 \mu (F(\theta^{(0)}) - F_*)$ for convenience, the second inequality uses \eqref{Eq: Chosen number of iterations} and the fact that $\epsilon \leq \frac{(L_1 - \mu)p}{2 L_1 \mu}$ (which is equivalent to $\frac{\sigma}{\sigma - 1} \leq \frac{p}{2 \mu \epsilon}$), and the third equality uses \eqref{Eq: Scaling in Main Theorem}. This completes the proof.
\qed
\end{proof}

We remark that the proofs of \propref{Prop: Number of Iterations} and \thmref{Thm: Oracle Complexity} can be combined to determine the optimal $t$-dependent approximation error parameters $\{\delta_t \in (0,1) : t \in \N\}$ and number of iterations $T \in \N$ such that $\theta^{(T)}$ is an $\epsilon$-approximate solution in the sense of \eqref{Eq: Minimum of convex function} with as few oracle calls as possible. However, this does not qualitatively improve the dominant term in the scaling law of the derived oracle complexity bound in \thmref{Thm: Oracle Complexity}. To see this, define the oracle complexity function, which takes the $\{\delta_t \in (0,1) : t \in \N\}$ and $T$ as inputs, and outputs the total number of first order oracle calls made by the LPI-GD algorithm after $T$ iterations:
\begin{equation}
G(\{\delta_t : t \in \N\},T) \triangleq \sum_{t=1}^{T}{|\G_{m_t}|} = \sum_{t=1}^{T}{m_t^d} \, ,
\end{equation}
where the grid sizes $\{m_t: t \in \N\}$ are defined in \eqref{Eq: Algo grid sizes}. Furthermore, propelled by \eqref{Eq: Bound on Diff to Optimality} in the proof of \propref{Prop: Number of Iterations}, define the constraint function
\begin{equation}
H(\{\delta_t : t \in \N\},T) \triangleq \left(1 - \frac{1}{\sigma}\right)^{\! T} \! \left(F(\theta^{(0)}) - F_*\right) + \frac{p}{2 L_1} \sum_{t = 1}^{T}{\left(1 - \frac{1}{\sigma}\right)^{\! T-t}\delta_t^2} \, ,
\end{equation}
which upper bounds the difference $F(\theta^{(T)}) - F_*$. Then, for any fixed $T \in \N$, the problem of finding the minimum oracle complexity of the LPI-GD algorithm to guarantee an $\epsilon$-approximate solution may be posed as:
\begin{equation}
\label{Eq: oracle complexity optimization}
\inf_{\substack{\{\delta_t \in (0,1) : t \in [T]\} :\\ H(\{\delta_t : t \in [T]\},T) \leq \epsilon}}{G(\{\delta_t : t \in [T]\},T)} \, .
\end{equation}
For simplicity, let us neglect the bounds on $\delta_t$ and assume that the inequality constraint is an equality constraint. The resulting optimization problem has the Lagrangian function
\begin{align}
\mathscr{L}(\{\delta_t : t \in \N\},\nu) = G(\{\delta_t : t \in [T]\},T) + \nu H(\{\delta_t : t \in \N\},T)
\end{align}
with a Lagrange multiplier $\nu \in \R$. For any $t \in [T]$, taking the partial derivative of $\mathscr{L}$ with respect to $\delta_t$ yields the stationarity condition $\frac{\partial \mathscr{L}}{\partial \delta_t} = 0$, which can be seen to be equivalent to
\begin{equation}
\delta_t^{2 + (d/\eta)} \propto \left(1 - \frac{1}{\sigma}\right)^{\! t} ,
\end{equation}
where the proportionality constant depends on other (fixed) problem parameters. Hence, the optimal sequence $\{\delta_t : t \in [T]\}$ decays geometrically as $t$ increases (since we assume that $\eta > d$). Moreover, since for all $t \in [T]$, the grid size $m_t^d \propto \delta_t^{-d/\eta}$ according to \eqref{Eq: Algo grid sizes}, the optimal sequence $\{m_t : t \in [T]\}$ increases exponentially fast as $t$ increases. Therefore, to obtain the dominant term in the scaling law of the minimum oracle complexity in \eqref{Eq: oracle complexity optimization}, it suffices to take $\delta_t$ to be constant in $t$ as we have in \eqref{Eq: Choices of the Chebyshev constants}; this also simplifies the details of the analysis.

\subsection{Proof of \propref{Prop: Comparison to GD and SGD}}
\label{Proof of Prop Comparison to GD and SGD}

Lastly, we derive \propref{Prop: Comparison to GD and SGD} from \thmref{Thm: Oracle Complexity}.

\begin{proof}[of \propref{Prop: Comparison to GD and SGD}]
We first establish the scaling of $C_{\mu,L_1,L_2,b,c}(d,l)$ in \eqref{Eq: Scaling in Main Theorem}. Since $d = O(\log\log(n))$, observe that
\begin{equation}
\label{Eq: Polylog term}
\left(\frac{\sigma + 2(L_1 - \mu)}{(L_1 - \mu)\log\!\left(\frac{\sigma}{\sigma - 1}\right)} \right) \! \left( \frac{220 (2 L_2 + 1) c}{b} \right)^{\! d} = O(\polylog(n)) \, , 
\end{equation}
and since $l \geq d$, observe that
$$ \frac{d^d (3e)^{d^2}}{\Lambda(d,l)^{2 d}} = O\!\left( (l+d)^{6 d l \left(\frac{e(l+d)}{d}\right)^{\! d}} \right) $$
as $d$ grows, where we use \eqref{Eq: Lower bound def}. To simplify the second scaling expression above, notice that for all sufficiently large $d \in \N$, we have
\begin{align*}
\log\log\!\left(\frac{d^d (3e)^{d^2}}{\Lambda(d,l)^{2 d}}\right) & \leq \log\log(l+d) + \log(dl) + d\log\!\left(\frac{e(l+d)}{d}\right) + \Theta(1) \\
& = \Theta(\log(d)) + d\log\!\left(\frac{e(l+d)}{d}\right) \\
& \leq \Theta(\log(d)) + d\log(e(\gamma + 1)) \\
& \leq 2 \log(e(\gamma + 1)) d
\end{align*}
where the second equality follows from the fact that $\eta = \Theta(d)$, and the third inequality holds because $l \leq \gamma d$. This implies that for all sufficiently large $d \in \N$,
\begin{equation}
\label{Eq: Quasi-polylog term}
\frac{d^d (3e)^{d^2}}{\Lambda(d,l)^{2 d}} \leq \exp\!\left(e^{2 \log(e(\gamma + 1)) d}\right) . 
\end{equation}
Hence, since $d \leq \frac{\log\log(n)}{4 \log(e(\gamma + 1))}$, we obtain
\begin{align*}
C_{\mu,L_1,L_2,b,c}(d,l) & \leq O(\polylog(n)) \cdot \exp\!\left(e^{\frac{1}{2}\! \log\log(n)}\right) \\
& = O(\polylog(n)) \cdot \exp\!\left(\!\sqrt{\log(n)}\right) \\
& \leq \exp\!\left(2 \sqrt{\log(n)}\right)
\end{align*}
for all sufficiently large $n \in \N$, using \eqref{Eq: Scaling in Main Theorem}, \eqref{Eq: Polylog term}, and \eqref{Eq: Quasi-polylog term}. Next, applying \thmref{Thm: Oracle Complexity}, we have that for all sufficiently large $n \in \N$,
\begin{align*}
\Gamma(\text{LPI-GD}) & \leq \exp\!\left(2 \sqrt{\log(n)}\right) \left(\frac{p + 2 \mu (F(\theta^{(0)}) - F_*)}{\epsilon}\right)^{\! d/(2\eta)} \log\!\left(\frac{p + 2 \mu (F(\theta^{(0)}) - F_*)}{2 \mu \epsilon}\right) \\
& \leq \exp\!\left(2 \sqrt{\log(n)}\right) \left(c_1 n^{\alpha + \beta}\right)^{\! d/(2\eta)} \log\!\left(c_2 n^{\alpha + \beta}\right) \\
& \leq c_3 \log(n) \exp\!\left(2 \sqrt{\log(n)}\right) n^{\! 1/\tau} 
\end{align*}
where $c_1,c_2,c_3 > 0$ are some positive constants, the second inequality follows from the assumptions that $\epsilon = \Theta(n^{-\alpha})$ and $p = O(n^{\beta})$, and the third inequality holds because $\eta \geq \tau (\alpha + \beta) d / 2$. This proves the scaling of $\Gamma(\text{LPI-GD})$ with respect to $n$.

Finally, from the relations \eqref{Eq: Gamma of GD} and \eqref{Eq: Gamma of SGD} and the assumption $\epsilon = \Theta(n^{-\alpha})$, we get $\Gamma_*(\text{GD}) = \Theta(n \log(n))$ and $\Gamma_*(\text{SGD}) = \Theta(n^{\alpha})$. Therefore, for some constants $c_4,c_5 > 0$, we obtain
$$ \lim_{n \rightarrow \infty}{\frac{\Gamma(\text{LPI-GD})}{\Gamma_*(\text{GD})}} \leq \lim_{n \rightarrow \infty}{\frac{c_3 \exp\!\left(2 \sqrt{\log(n)}\right) n^{\! 1/\tau} }{c_4 n}} = 0 \, , $$
$$ \lim_{n \rightarrow \infty}{\frac{\Gamma(\text{LPI-GD})}{\Gamma_*(\text{SGD})}} \leq \lim_{n \rightarrow \infty}{\frac{c_3 \log(n) \exp\!\left(2 \sqrt{\log(n)}\right) n^{\! 1/\tau}}{c_5 n^{\alpha}}} = 0 \, , $$
which use the fact that $\tau > \max\{1,\alpha^{-1}\}$. This completes the proof.
\qed
\end{proof}

\section{Conclusion}
\label{Conclusion}

We conclude by recapitulating our main contributions and reiterating some directions of future research. In order to exploit the smoothness of loss functions in data for optimization purposes, we proposed the LIP-GD algorithm (see \algoref{Algorithm: LPI-GD}) to compute approximate solutions of the ERM problem in \eqref{Eq: ERM Objective Function} and \eqref{Eq: ERM}, where the loss function is strongly convex and smooth in both the parameter and the data. We then derived the iteration and oracle complexities of this algorithm, and illustrated that its oracle complexity beats the oracle complexities of both GD and SGD for a very broad range of dependencies between the parameter dimension $p$, the approximation accuracy $\epsilon$, and the number of training samples $n$, when the data dimension $d$ is sufficiently small. Finally, in the course of our convergence analysis, we also provided a careful and detailed analysis of multivariate local polynomial interpolation with supremum norm guarantees, which may be of independent interest in non-parametric statistics.

We close our discussion by stating three avenues for future work. Firstly, as stated at the end of \secref{Main Results and Discussion} and in \secref{Proof of Prop Lowner Lower Bound}, it would be very valuable to improve the dependence of \thmref{Thm: Oracle Complexity} on the data dimension $d$ by sharpening the minimum eigenvalue lower bound in the key technical estimate in \lemref{Lemma: Auxiliary Loewner Lower Bound}. Secondly, the dependence of \thmref{Thm: Oracle Complexity} on the parameter dimension $p$ could be improved by establishing new interpolation guarantees for approximating smooth $\R^p$-valued functions using local polynomial regression or other methods. Lastly, as noted in \secref{Introduction}, while our work focuses on the strongly convex loss function setting, similar analyses could be carried out for other conventional scenarios, e.g., when $f(x;\cdot):\R^p \rightarrow \R$ is convex, or when $f(x;\cdot):\R^p \rightarrow \R$ is non-convex and $\epsilon$-approximate solutions to \eqref{Eq: ERM} are defined via approximate first order stationary points.

\appendix

\section{Taylor Approximation of H\"{o}lder Class Functions}
\label{Taylor Approximation of Holder Class Functions}

In this appendix, we present a useful upper bound on the Taylor approximation error for functions residing in H\"{o}lder classes. While such results are known in the literature (see, e.g., \cite[Chapter ``Density Estimation'' Equation (12), Chapter ``Nonparametric Regression'' Equation (2)]{Wasserman2019}), the ensuing proposition is adapted to our specific setting.

\begin{proposition}[Taylor Approximation of H\"{o}lder Class Functions]
\label{Prop: Taylor Approximation of Holder Class Functions}
For any function $g : [0,1]^d \rightarrow \R$ in the $(\eta,L_2)$-H\"{o}lder class, we have
$$ \forall x,u \in [0,1]^d, \enspace \left|g(x) - \sum_{s \in \Z_+^d : \, |s| \leq l}{\frac{\nabla^s g (u)}{s!} (x - u)^s} \right| \leq \frac{L_2}{l!} \left\|x - u\right\|_1^{\eta}$$
where $l = \lceil \eta \rceil - 1$, and we utilize multi-index notation.
\end{proposition}

\begin{proof} 
By Taylor's theorem \cite[Theorem 12.14]{Apostol1974}, for every $x,u \in [0,1]^d$,
$$ g(x) = \sum_{s \in \Z_+^d : \, |s| \leq l-1}{\frac{\nabla^s g (u)}{s!} (x - u)^s} + \sum_{s \in \Z_+^d : \, |s| = l}{\frac{\nabla^s g (u + \tau (x - u))}{s!} (x - u)^s} $$
for some $\tau \in (0,1)$, where the second summation on the right hand side is the Lagrange remainder term. This implies that
\begin{align*}
\left|g(x) - \sum_{s \in \Z_+^d : \, |s| \leq l}{\frac{\nabla^s g (u)}{s!} (x - u)^s}\right| & = \left| \sum_{s \in \Z_+^d : \, |s| = l}{\frac{\left(\nabla^s g (u + \tau (x - u)) - \nabla^s g (u)\right)}{s!} (x - u)^s} \right| \\
& \leq \sum_{s \in \Z_+^d : \, |s| = l}{\frac{\left|\nabla^s g (u + \tau (x - u)) - \nabla^s g (u)\right|}{s!} \left|(x - u)^s\right|} \\
& \leq L_2 \left\|x - u \right\|_1^{\eta - l} \sum_{s \in \Z_+^d : \, |s| = l}{\frac{1}{s!} \left|(x - u)^s\right|} \\
& = \frac{L_2}{l!} \left\|x - u \right\|_1^{\eta} ,
\end{align*}
where the second inequality follows from the triangle inequality, the third inequality follows from the H\"{o}lder class assumption (see \eqref{Eq: Holder condition}) and the fact that $\tau \in (0,1)$, and the last equality follows from the multinomial theorem. This completes the proof.
\qed
\end{proof}

\begin{acknowledgements}
This research was supported in part by the Vannevar Bush Fellowship, in part by the ARO Grant W911NF-18-S-0001, in part by the NSF CIMS 1634259, in part by the NSF CNS 1523546, and in part by the MIT-KACST project.
\end{acknowledgements}

% References section
\bibliographystyle{spmpsci} % author-year citations
\bibliography{LPIGDrefs}

\begin{thebibliography}{10}
\providecommand{\url}[1]{{#1}}
\providecommand{\urlprefix}{URL }
\expandafter\ifx\csname urlstyle\endcsname\relax
  \providecommand{\doi}[1]{DOI~\discretionary{}{}{}#1}\else
  \providecommand{\doi}{DOI~\discretionary{}{}{}\begingroup
  \urlstyle{rm}\Url}\fi

\bibitem{Agarwaletal2009}
Agarwal, A., Wainwright, M.J., Bartlett, P.L., Ravikumar, P.K.:
  Information-theoretic lower bounds on the oracle complexity of convex
  optimization.
\newblock In: Proceedings of the Advances in Neural Information Processing
  Systems 22 (NIPS), pp. 1--9. Vancouver, BC, Canada (2009)

\bibitem{Apostol1974}
Apostol, T.M.: Mathematical Analysis, second edn.
\newblock Addison-Wesley, Reading, MA, USA (1974)

\bibitem{AussetClemenconPortier2020}
Ausset, G., Cl\'{e}men\c{c}on, S., Portier, F.: Nearest neighbour based
  estimates of gradients: {S}harp nonasymptotic bounds and applications (2020).
\newblock \urlprefix\url{https://arxiv.org/abs/2006.15043}.
\newblock ArXiv:2006.15043 [cs.LG]

\bibitem{BelkinNiyogi2001}
Belkin, M., Niyogi, P.: Laplacian eigenmaps and spectral techniques for
  embedding and clustering.
\newblock In: Proceedings of the Advances in Neural Information Processing
  Systems 14 (NIPS), pp. 585--591. Vancouver, BC, Canada (2001)

\bibitem{Bishop2006}
Bishop, C.M.: Pattern Recognition and Machine Learning.
\newblock Information Science and Statistics. Springer, New York, NY, USA
  (2006)

\bibitem{Bubeck2015}
Bubeck, S.: Convex Optimization: {A}lgorithms and Complexity, \emph{Foundations
  and Trends in Machine Learning}, vol.~8.
\newblock now Publishers Inc., Hanover, MA, USA (2015)

\bibitem{Carmonetal2019}
Carmon, Y., Duchi, J.C., Hinder, O., Sidford, A.: Lower bounds for finding
  stationary points {II}: first-order methods.
\newblock Mathematical Programming, Series A pp. 1--41 (2019)

\bibitem{Chihara2011}
Chihara, T.S.: An Introduction to Orthogonal Polynomials.
\newblock Dover, New York, NY, USA (2011)

\bibitem{Choietal2018}
Choi, E., Xiao, C., Stewart, W.F., Sun, J.: \texttt{MiME}: {M}ultilevel medical
  embedding of electronic health records for predictive healthcare.
\newblock In: Proceedings of the Advances in Neural Information Processing
  Systems 32 (NeurIPS), pp. 4547--4557. Montr\'{e}al, QC, Canada (2018)

\bibitem{ClevelandLoader1996}
Cleveland, W.S., Loader, C.: Smoothing by local regression: {P}rinciples and
  methods.
\newblock In: W.~H\"{a}rdle, M.G. Schimek (eds.) Statistical Theory and
  Computational Aspects of Smoothing: {P}roceedings of the {COMPSTAT} '94
  Satellite Meeting held in {S}emmering, {A}ustria 27-28 {A}ugust 1994,
  Contributions to Statistics, pp. 10--49. Physica-Verlag, Heidelberg, Germany
  (1996)

\bibitem{CoifmanLafon2006}
Coifman, R.R., Lafon, S.: Diffusion maps.
\newblock Applied and Computational Harmonic Analysis, Elsevier \textbf{21}(1),
  5--30 (2006)

\bibitem{DasguptaGupta2003}
Dasgupta, S., Gupta, A.: An elementary proof of a theorem of {J}ohnson and
  {L}indenstrauss.
\newblock Random Structures and Algorithms \textbf{22}(1), 60--65 (2003)

\bibitem{DeBrabanteretal2013}
{De Brabanter}, K., {De Brabanter}, J., {De Moor}, B., Gijbels, I.: Derivative
  estimation with local polynomial fitting.
\newblock Journal of Machine Learning Research \textbf{14}(1), 281--301 (2013)

\bibitem{Dekeletal2012}
Dekel, O., Gilad-Bachrach, R., Shamir, O., Xiao, L.: Optimal distributed online
  prediction using mini-batches.
\newblock Journal of Machine Learning Research \textbf{13}(6), 165--202 (2012)

\bibitem{DelecroixRosa1996}
Delecroix, M., Rosa, A.C.: Nonparametric estimation of a regression function
  and its derivatives under an ergodic hypothesis.
\newblock Journal of Nonparametric Statistics \textbf{6}(4), 367--382 (1996)

\bibitem{DuttonBrigham1986}
Dutton, R.D., Brigham, R.C.: Computationally efficient bounds for the {C}atalan
  numbers.
\newblock European Journal of Combinatorics \textbf{7}(3), 211--213 (1986)

\bibitem{FanGijbels1996}
Fan, J., Gijbels, I.: Local Polynomial Modeling and Its Applications,
  \emph{Monographs on Statistics and Applied Probability}, vol.~66.
\newblock Chapman and Hall, London, UK (1996)

\bibitem{Feller1968}
Feller, W.: An Introduction to Probability Theory and Its Applications, vol.~1,
  third edn.
\newblock John Wiley \& Sons, Inc., New York, NY, USA (1968)

\bibitem{Forbesetal2011}
Forbes, C., Evans, M., Hastings, N., Peacock, B.: Statistical Distributions,
  fourth edn.
\newblock John Wiley \& Sons, Inc., Hoboken, NJ, USA (2011)

\bibitem{FriedlanderSchmidt2012}
Friedlander, M.P., Schmidt, M.: Hybrid deterministic-stochastic methods for
  data fitting.
\newblock SIAM Journal on Scientific Computing \textbf{34}(3), A1380--A1405
  (2012)

\bibitem{HastieTibshiraniFriedman2009}
Hastie, T., Tibshirani, R., Friedman, J.: The Elements of Statistical Learning:
  {D}ata Mining, Inference, and Prediction, second edn.
\newblock Springer Series in Statistics. Springer, New York, NY, USA (2009)

\bibitem{HornJohnson2013}
Horn, R.A., Johnson, C.R.: Matrix Analysis, second edn.
\newblock Cambridge University Press, New York, NY, USA (2013)

\bibitem{Hotelling1933}
Hotelling, H.: Analysis of a complex of statistical variables into principal
  components.
\newblock Journal of Educational Psychology \textbf{24}(6), 417--441 and
  498--520 (1933)

\bibitem{Hotelling1936}
Hotelling, H.: Relations between two sets of variates.
\newblock Biometrika \textbf{28}(3/4), 321--377 (1936)

\bibitem{Huangetal2019}
Huang, S.L., Makur, A., Wornell, G.W., Zheng, L.: On universal features for
  high-dimensional learning and inference (2019).
\newblock \urlprefix\url{https://arxiv.org/abs/1911.09105}.
\newblock ArXiv:1911.09105 [cs.LG]

\bibitem{JohnsonZhang2013}
Johnson, R., Zhang, T.: Accelerating stochastic gradient descent using
  predictive variance reduction.
\newblock In: Proceedings of the Advances in Neural Information Processing
  Systems 26 (NIPS), pp. 315--323. Lake Tahoe, NV, USA (2013)

\bibitem{KrantzParks2002}
Krantz, S.G., Parks, H.R.: A Primer of Real Analytic Functions, second edn.
\newblock Birkh\"{a}user Advanced Texts: Basler Lehrb\"{u}cher. Birkh\"{a}user,
  Boston, MA, USA (2002)

\bibitem{Makur2019}
Makur, A.: Information contraction and decomposition.
\newblock Sc.{D}. thesis in {E}lectrical {E}ngineering and {C}omputer
  {S}cience, Massachusetts Institute of Technology, Cambridge, MA, USA (2019)

\bibitem{MarshallOlkinArnold2011}
Marshall, A.W., Olkin, I., Arnold, B.C.: Inequalities: Theory of Majorization
  and Its Applications, second edn.
\newblock Springer Series in Statistics. Springer, New York, NY, USA (2011)

\bibitem{MerikoskiVirtanen1997}
Merikoski, J.K., Virtanen, A.: Bounds for eigenvalues using the trace and
  determinant.
\newblock Linear Algebra and its Applications, Elsevier \textbf{264}, 101--108
  (1997)

\bibitem{MukherjeeWu2006}
Mukherjee, S., Wu, Q.: Estimation of gradients and coordinate covariation in
  classification.
\newblock Journal of Machine Learning Research \textbf{7}(88), 2481--2514
  (2006)

\bibitem{MukherjeeZhou2006}
Mukherjee, S., Zhou, D.X.: Learning coordinate covariances via gradients.
\newblock Journal of Machine Learning Research \textbf{7}(18), 519--549 (2006)

\bibitem{NechybaXu1994}
Nechyba, M.C., Xu, Y.: Neural network approach to control system identification
  with variable activation functions.
\newblock In: Proceedings of the 9th IEEE International Symposium on
  Intelligent Control, pp. 358--363. Columbus, OH, USA (1994)

\bibitem{Nedic2015}
Nedi\'{c}, A.: Distributed optimization.
\newblock In: J.~Baillieul, T.~Samad (eds.) Encyclopedia of Systems and
  Control, pp. 308--317. Springer, London, UK (2015)

\bibitem{Nelles2001}
Nelles, O.: Nonlinear System Identification: {F}ron Classical Approaches to
  Neural Networks and Fuzzy Models.
\newblock Springer, Berlin, Heidelberg, Germany (2001)

\bibitem{Nemirovskietal2009}
Nemirovski, A., Juditsky, A., Lan, G., Shapiro, A.: Robust stochastic
  approximation approach to stochastic programming.
\newblock SIAM Journal on Optimization \textbf{19}(4), 1574--1609 (2009)

\bibitem{NemirovskiiYudin1983}
Nemirovski\u{\i}, A.S., Yudin, D.B.: Problem Complexity and Method Efficiency
  in Optimization.
\newblock Wiley-Interscience Series in Discrete Mathematics and Optimization.
  John Wiley \& Sons Inc., New York, NY, USA (1983).
\newblock {T}ranslated from Russian by E. R. Dawson

\bibitem{Nesterov2004}
Nesterov, Y.: Introductory Lectures on Convex Optimization: {A} Basic Course,
  \emph{Applied Optimization}, vol.~87.
\newblock Springer, New York, NY, USA (2004)

\bibitem{Nesterov1983}
Nesterov, Y.E.: A method of solving a convex programming problem with
  convergence rate {$O\big(\frac{1}{k^2}\big)$}.
\newblock Doklady Akademii Nauk SSSR \textbf{269}(3), 543--547 (1983).
\newblock In Russian

\bibitem{Netrapalli2019}
Netrapalli, P.: Stochastic gradient descent and its variants in machine
  learning.
\newblock Journal of Indian Institute of Science \textbf{99}(2), 201--213
  (2019)

\bibitem{Pearson1901}
Pearson, K.: On lines and planes of closest fit to systems of points in space.
\newblock Philosophical Magazine \textbf{2}(11), 559--572 (1901)

\bibitem{PoggioBanburskiLiao2020}
Poggio, T., Banburski, A., Liao, Q.: Theoretical issues in deep networks.
\newblock Proceedings of the National Academy of Sciences of the United States
  of America (PNAS) pp. 1--7 (2020)

\bibitem{SoZhou2017}
So, A.M.C., Zhou, Z.: Non-asymptotic convergence analysis of inexact gradient
  methods for machine learning without strong convexity.
\newblock Optimization Methods and Software \textbf{32}(4), 963--992 (2017)

\bibitem{Tsybakov2009}
Tsybakov, A.B.: Introduction to Nonparametric Estimation.
\newblock Springer Series in Statistics. Springer, New York, NY, USA (2009)

\bibitem{Wangetal2018}
Wang, Y., Du, S.S., Balakrishnan, S., Singh, A.: Stochastic zeroth-order
  optimization in high dimensions.
\newblock In: Proceedings of the 21st International Conference on Artificial
  Intelligence and Statistics (AISTATS), pp. 1356--1365. Lanzarote, Spain
  (2018)

\bibitem{Wasserman2019}
Wasserman, L.: Statistical methods for machine learning (2019).
\newblock {D}epartment of Statistics and Data Science, CMU, Pittsburgh, PA,
  USA, Lecture Notes 36-708

\bibitem{ZhouWolfe2000}
Zhou, S., Wolfe, D.A.: On derivative estimation in spline regression.
\newblock Statistica Sinica \textbf{10}(1), 93--108 (2000)

\end{thebibliography}

\end{document}